\documentclass[conference,9pt]{IEEEtran}

\AtBeginDocument{%
  \providecommand\BibTeX{{%
    \normalfont B\kern-0.5em{\scshape i\kern-0.25em b}\kern-0.8em\TeX}}}

 \usepackage{pifont}
\usepackage{tabularx}
\usepackage{booktabs} 
\usepackage{adjustbox}
\usepackage{amsthm}
\usepackage{amsmath}
\usepackage{amssymb}
\usepackage{url}
\usepackage{enumitem}

\usepackage[utf8]{inputenc} 
\usepackage[T1]{fontenc}    
\usepackage{hyperref}       
\usepackage{url}            
\usepackage{booktabs}       
\usepackage{amsfonts}       
\usepackage{nicefrac}       
\usepackage{microtype}      
\usepackage{xcolor}         
\usepackage{colortbl}
\usepackage{subcaption}
\usepackage{graphicx}
\usepackage{svg}
\usepackage{multirow}
\usepackage{wrapfig}
\usepackage{threeparttable}
\usepackage{tikz}
\usepackage{makecell}
\usepackage{tablefootnote}
\usepackage{cite}
\usepackage{color}
\usepackage{amsmath}
\usepackage{float}

\usepackage{xcolor}
\definecolor{mycustomcolor}{rgb}{0.800, 0.400, 0.000}
\definecolor{mycommentcolor}{rgb}{0.0, 0.5, 0.0}
\usepackage{tikz}
\usepackage{listings}
\usepackage{multirow}
\usepackage{subcaption}
\usepackage{pifont}
\usepackage{tabularx}
\usepackage{booktabs} 
\usepackage{adjustbox}
\usepackage{amsmath}
\usepackage{makecell}
\usepackage{url}
\usepackage{enumitem}

\usepackage[utf8]{inputenc} 
\usepackage[T1]{fontenc}    
\usepackage{hyperref}       
\usepackage{url}            
\usepackage{booktabs}       
\usepackage{amsfonts}       
\usepackage{nicefrac}       
\usepackage{microtype}      
\usepackage{xcolor}         
\usepackage{colortbl}
\usepackage{subcaption}
\usepackage{graphicx}
\usepackage{svg}
\usepackage{multirow}
\usepackage{wrapfig}
\usepackage{threeparttable}
\usepackage{tikz}
\usepackage{makecell}
\usepackage{tablefootnote}
\usepackage{color}
\usepackage{float}
\usepackage{diagbox}
\usepackage{tabularx}
\usepackage{algorithm}
\usepackage{xspace}
\usepackage{algpseudocode}
\usepackage[strict]{changepage}
\usepackage{courier}
\usepackage{amsmath}

\usepackage{xcolor}              
\definecolor{mygreen}{HTML}{2ca25f}

\newcommand{\cunxi}[1]{\textcolor{green}}

\newcommand{\name}{\texttt{e-boost}\xspace}

\begin{document}

\title{\texttt{\bf e-boost}: Boosted E-Graph Extraction with Adaptive Heuristics and Exact Solving}

\newcommand{\red}[1]{\textcolor{red}{#1}}
\newtheorem*{theorem}{Theorem}

\newcommand{\zhan}[1]{\textcolor{blue}{[\small zhan: ~#1]}}

\author{\IEEEauthorblockN{Jiaqi Yin$^{*}$, Zhan Song$^{*}$, Chen Chen$^{*}$, Yaohui Cai$^{\dagger}$, Zhiru Zhang$^{\dagger}$, Cunxi Yu$^{*}$}
\IEEEauthorblockA{
$^{*}$University of Maryland, College Park, MD, US \quad 
$^{\dagger}$Cornell University, Ithaca, NY, US \\
{\{jyin629,cunxiyu\}@umd.edu}}
}



\maketitle

\begin{abstract}

E-graphs have attracted growing interest in many fields, particularly in logic synthesis and formal verification. E-graph extraction is a challenging NP-hard combinatorial optimization problem. It requires identifying optimal terms from exponentially many equivalent expressions, serving as the primary performance bottleneck in e-graph based optimization tasks. However, traditional extraction methods face a critical trade-off: heuristic approaches offer speed but sacrifice optimality, while exact methods provide optimal solutions but face prohibitive computational costs on practical problems. We present \name, a novel framework that bridges this gap through three key innovations: (1) parallelized heuristic extraction that leverages weak data dependence to compute DAG costs concurrently, enabling efficient multi-threaded performance without sacrificing extraction quality; (2) adaptive search space pruning that employs a parameterized threshold mechanism to retain only promising candidates, dramatically reducing the solution space while preserving near-optimal solutions; and (3) initialized exact solving that formulates the reduced problem as an Integer Linear Program with warm-start capabilities, guiding solvers toward high-quality solutions faster.

Across the diverse benchmarks in formal verification and logic synthesis fields, \name demonstrates $558\times$ runtime speedup over traditional exact approaches (ILP) and 19.04\% performance improvement over the state-of-the-art extraction framework (SmoothE). In realistic logic synthesis tasks, \name produces 7.6\% and 8.1\% area improvements compared to conventional synthesis tools with two different technology mapping libraries. \name is available at \url{https://github.com/Yu-Maryland/e-boost}.

\end{abstract}

\section{Introduction}


{Equality saturation, an emerging optimization approach, has gained increasing interest in Electronic Design Automation (EDA), particularly for tackling complex problems in logic synthesis \cite{ustun2022impress,chen2025emorphicscalableequalitysaturation, ustun2023equality} and formal verification \cite{coward2023datapath,yin2025boole,yin2025hec}. This technique explores vast spaces of equivalent expressions by iteratively applying rewrite rules within an e-graph data structure \cite{tate2009equality}. Beyond EDA, its effectiveness is also recognized in fields such as compiler optimization\cite{cheng2023seer} and theorem proving \cite{detlefs2005simplify}. However, e-graph extraction, a critical step in this process, poses an \textit{NP-hard} combinatorial optimization challenge. This challenge arises from the need to select optimal terms from a potentially exponential number of equivalent expressions within e-classes, particularly when employing realistic DAG-based cost \cite{mishchenko2006dag,li2024dag,yu2016dag} models that capture shared subexpressions and introduce a trade-off between solution quality and computational efficiency.}


Several approaches have been developed to tackle the {e-graph extraction challenge}.
Exact methods, often formulated as Integer Linear Programs (ILP) \cite{yang2021equality}, can guarantee optimal solutions but suffer from significant scalability issues, frequently exceeding realistic time limits on larger problem instances. Heuristic techniques \cite{panchekha2015automatically}, such as greedy algorithms, provide a faster alternative but sacrifice optimality, potentially leading to suboptimal results due to their localized decision-making. {More recently, Cai et al. proposed a novel differentiable approach and introduced SmoothE \cite{cai2025smoothe}. SmoothE translates the discrete problem into a continuous optimization search space and optimizes e-graph extraction by gradient descent. However, the differentiable methods can struggle with convergence speed on complex e-graphs, risk getting stuck in local optima, and are highly dependent on GPU acceleration for practical implementation.}

To address these challenges, we introduce \name, a novel framework designed for efficient and high-quality e-graph extraction. Taking a saturated e-graph as input, \name generates optimized extraction results, aiming to provide a scalable solution that delivers near-optimal quality. The core idea behind \name is to bridge the gap between fast heuristics and optimal exact solvers. It leverages the speed of heuristic methods to prune the vast search space and provide a high-quality warm start for a subsequent exact solving phase. This combination allows \name to focus the power of exact methods on a dramatically reduced, more tractable search space, thereby achieving significantly improved performance and scalability without substantial compromises in solution quality.

The main contributions of this paper are summarized as follows:
\begin{itemize}
    \item {\textbf{A hybrid heuristic-exact extraction framework -- \name}: We introduce \name, a novel framework designed to harness the complementary strengths of heuristic search and exact optimization. By employing adaptive heuristics for rapid, large-scale search space reduction and warm-start initialization, and then deploying exact solvers (ILP) for precise optimization within the pruned space, \name achieves a powerful combination of methods. This results in substantial runtime improvements while maintaining near-optimal solution quality. The \name framework operates fully automatically without manual tuning, and is available as open-source.}
    \item {\textbf{Accelerated heuristic pruning via parallelization}: To enhance the \name framework, we introduce a parallelized heuristic algorithm specifically designed to speed up the critical search space pruning and warm-start generation steps. By leveraging multi-threading on weak data dependencies, this approach achieves substantial runtime reductions for the heuristic phase compared to a sequential execution, without compromising the quality of the results.}
    \item {\textbf{Evaluation on diverse benchmarks to demonstrate \name performance and scalability:} \name demonstrates significant advantages: it achieves $19.04\%$ higher performance (with $138\times$ speedup) compared to the state-of-the-art (SOTA) differentiable framework SmoothE\cite{cai2025smoothe}, and obtains 5.6\% better quality ($558\times$ speedup) compared to exact methods via \textbf{SOTA commercial solvers} such as Gurobi\cite{gurobi_2025}, IBM CPLEX\cite{ibm_ilog_cplex_optimization_studio_2025}, and Google CP-SAT\cite{google_cp_sat_2024}.
    }
    \item {\textbf{Our approach brings realistic benefits across various domains:}
    To demonstrate the practical impact of our approach, we use logic synthesis as a key example. Integrating \name into E-Syn\cite{chen2024syn}, an e-graph based logic synthesis tool, yields significant benefits, reducing circuit area by 7.6\% and 8.1\% when targeting the ASAP 7nm\cite{clark2016asap7} and Skywater 130nm\cite{gdsfactory_skywater130_2025} technology libraries, respectively, compared to the baseline flow.}
\end{itemize}
\section{Preliminary}

\subsection{E-graph}
An \textbf{e-graph} (or equivalence graph) serves as a data structure specifically designed to efficiently represent congruence relations among expressions \cite{tate2009equality,egglog,willsey2021egg}. Formally, an e-graph $\mathcal{E}$ is defined as a tuple $\mathcal{E} = (N, C, \lambda)$, where:
\begin{itemize}
    \item $N$ is a finite set of \textbf{e-nodes}, each representing an expression or sub-expression with a specific operator and children e-classes.
    \item $C$ is a set of \textbf{e-classes}, where each e-class contains e-nodes that are considered equivalent.
    \item $\lambda: N \rightarrow \Sigma \times C^*$ is a function mapping each e-node to:
    \begin{itemize}
        \item An operator symbol from the set $\Sigma$
        \item A sequence of children e-classes $C^*$ that function as the operator arguments
    \end{itemize}
\end{itemize}


To illustrate this concept, consider Figure~\ref{fig:motivations}, which shows an e-graph representation of the expression $\neg a \land (a \lor \neg b)$. In this figure, each node represents an \textbf{e-node} with a specific operator and operands, and its index is shown as a superscript. In this paper, we use $E_k$ to denote the e-node with index $k$. The dotted boxes represent \textbf{e-classes}, which group e-nodes that are considered equivalent. For example, $a$ ($E_6$) and $\neg(\neg a)$ ($E_5$) are grouped into the same e-class because they are logically equivalent. This compact representation allows a single e-graph to represent multiple equivalent expressions simultaneously. For instance, $E_2$ represents the group of equivalent expressions $a \lor \neg b$, $\neg(\neg a) \lor \neg b$, $a \lor \neg(\neg(\neg b))$, and $\neg(\neg a) \lor \neg(\neg(\neg b))$. E-graph extraction refers to the process of selecting the optimal expression from this rich set of equivalent representations according to different cost functions.

\subsection{Equality Saturation}
\textbf{Equality saturation} \cite{tate2009equality, egglog,willsey2021egg} is an optimization technique that uses e-graphs to explore equivalent expressions systematically. The process begins by creating an initial e-graph containing the input expression. During the saturation phase, the system iteratively applies rewrite rules to discover new equivalent expressions using e-matching to find patterns modulo equality. When a match is found, the corresponding transformation is added to the e-graph, creating new equivalences. This process continues until either no new equivalences can be added (a fixed point is reached) or a resource limit is reached. Finally, in the extraction phase, the system selects the "best" expression from the saturated e-graph according to a cost function. \textbf{This extraction phase is pivotal, as the selection directly determines the resulting expression and thus dictates the ultimate quality achieved by the entire equality saturation process, making it a primary focus of this work.} 
The \textbf{egg} framework\cite{willsey2021egg} implements efficient e-graphs and e-matching for equality saturation. Since its release, equality saturation has sparked new research interest across multiple domains, including compiler optimization \cite{tate2009equality}, hardware design automation \cite{ustun2022impress,coward2022automatic}, theorem proving \cite{detlefs2005simplify,de2008z3}, and various other applications \cite{yang2021equality,cao2023babble,vanhattum2021vectorization,yin2025hec,Smith_2021,wu2023gamora,deng2024less, yin2025boole}.

\subsection{Combinatorial nature of e-graph extraction}

Combinatorial optimization involves finding optimal solutions from a finite set of possibilities where exhaustive search is impractical. These problems typically seek to minimize or maximize objective functions under constraints, with solutions represented as discrete structures like permutations, assignments, or graphs. 
In computer science, these problems underpin critical applications including compiler optimization\cite{lozano2019combinatorial} and high-level synthesis scheduling\cite{lozano2019combinatorial,yin2023respect,yin2023accelerating,liu2024differentiable}. In electronic design automation, combinatorial optimization drives placement and routing, and logic synthesis \cite{malik1990combinational}. 


E-graph extraction \cite{stepp2011equality,goharshady2024fast} is a particularly challenging combinatorial optimization problem. After equality saturation generates the e-graph with equivalent terms, extraction identifies the optimal term from exponentially many expressions for each e-class according to a cost function. The primary extraction cost models include tree costs, and DAG costs \cite{flatt2022small}. While tree-based extraction offers the advantage of polynomial-time complexity, allowing for faster computation, it sacrifices accuracy in reflecting realistic implementation costs. This inaccuracy arises from the tree cost model counts shared subexpressions multiple times, potentially leading to suboptimal choices in practice. For example, in technology mapping, tree costs significantly overestimate actual circuit costs by counting shared logic gates multiple times, resulting in suboptimal designs. On the other hand, DAG cost models address this issue by counting each unique subexpression exactly once, regardless of how many times it is used. Consequently, DAG cost models offer a more accurate representation of actual costs in scenarios like logic synthesis or formal verification. However, DAG-based extraction is intrinsically \textit{NP-hard} \cite{zhang_egraph_extraction_2023}, creating a challenging combinatorial optimization problem that has limited the adoption of equality saturation in practice. Our work focuses on developing an efficient system for this more realistic but computationally challenging DAG-based extraction problem.

Traditional e-graph extraction approaches primarily rely on Integer Linear Programming (ILP) \cite{yang2021equality} and various heuristics \cite{panchekha2015automatically} to select optimal terms. ILP formulations provide optimal solutions but scale poorly with e-graph size, becoming prohibitively expensive for practical applications. Meanwhile, heuristic approaches like greedy traversal offer better performance but frequently produce suboptimal results, especially when faced with complex sharing patterns. {Recently, Cai et al. \cite{cai2025smoothe} introduced SmoothE, a differentiable approach that converts discrete e-node selection into an optimization problem over continuous probabilistic variables, enabling the use of GPU-accelerated gradient descent. However, this approach presents challenges: convergence often requires more iterations for large, complex e-graphs, the process is prone to getting trapped in local optima, and reliance on specialized computational resources (GPUs) for practical implementation.} All of these methods struggle to balance optimality with computational efficiency. This fundamental trade-off has significantly limited the practical adoption of equality saturation in real-world applications.


\section{Motivating Example}

\label{sec:motivating}


\begin{figure}[!htb]
    \centering
    \includegraphics[width=0.9\linewidth]{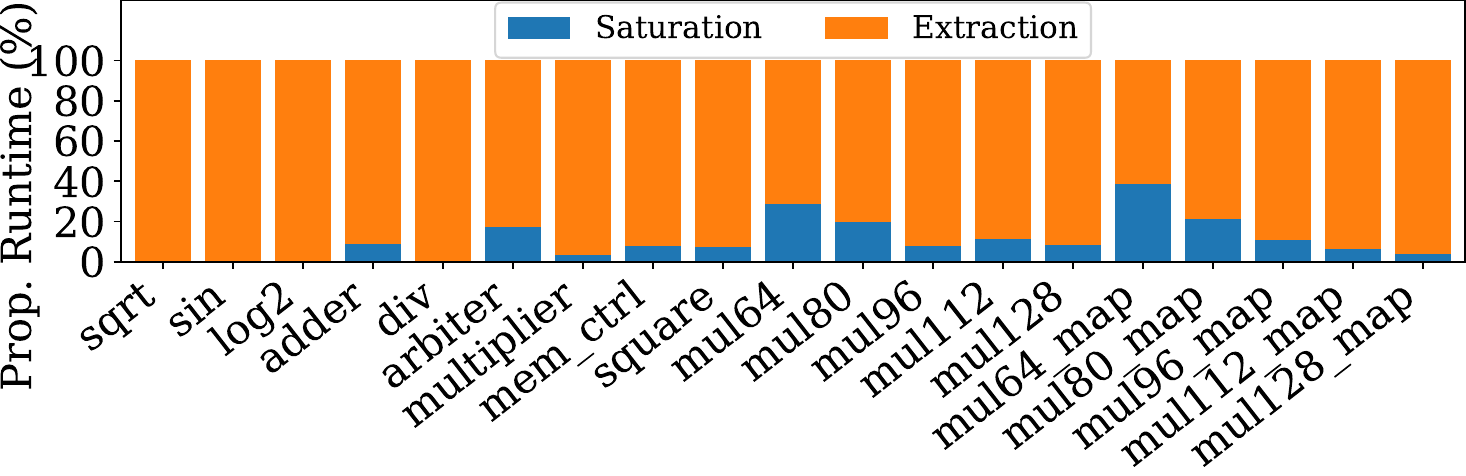}
    \vspace{-1mm}
    \caption{Normalized runtime proportions: saturation vs. extraction 
    }
    \label{fig:prop_runtime}
    \vspace{-3mm}
\end{figure}

\begin{figure*}[htbp]
    \centering
    \begin{subfigure}[t]{0.32\linewidth}
        \centering
        \includegraphics[width=\textwidth]{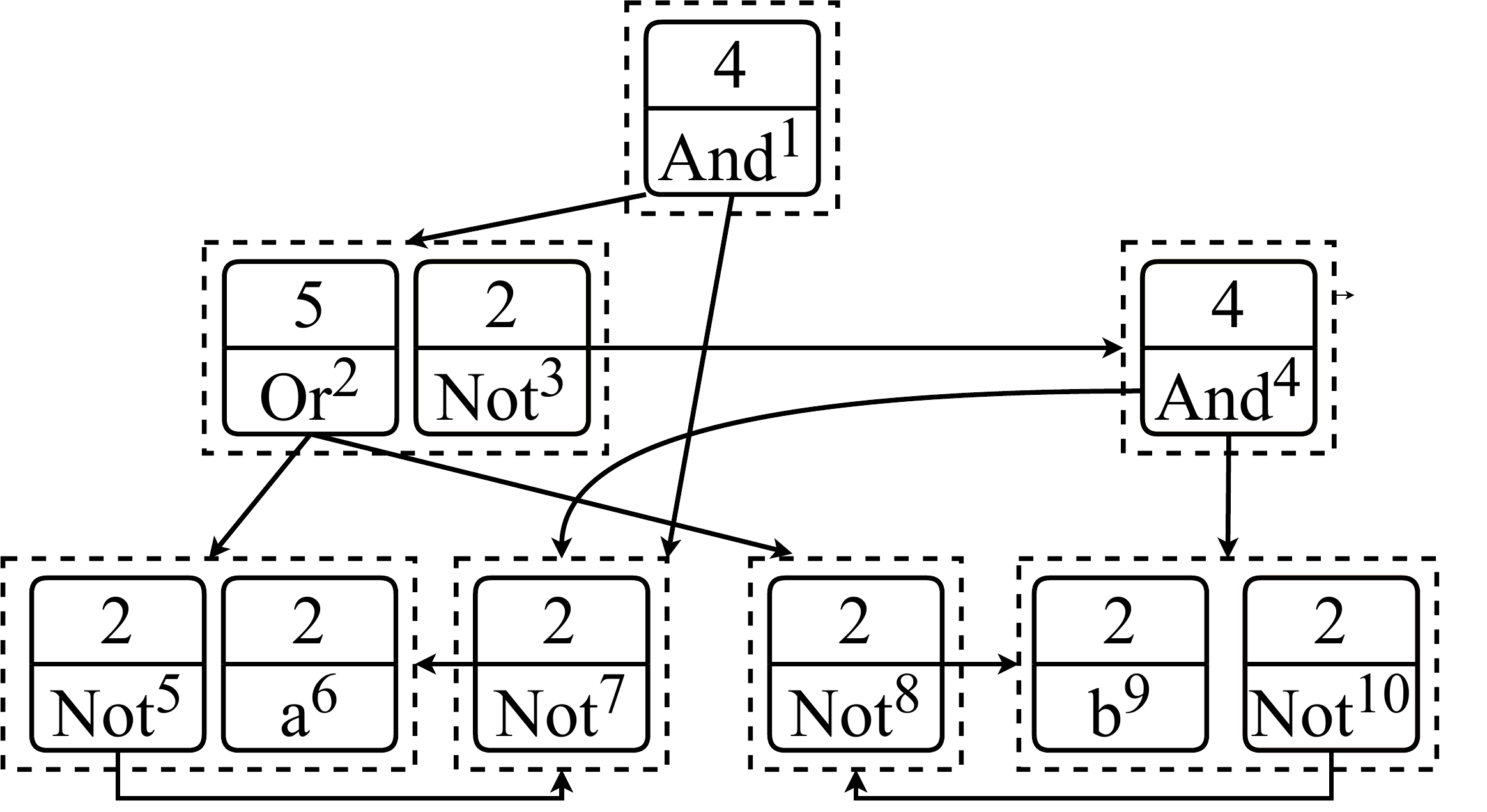}
        \caption{Initial e-graph representation of $\lnot a \wedge (a \vee \lnot b)$. Each e-node is annotated with its node cost (displayed within the node) and its index (shown in superscript). Primary inputs $a$ and $b$ are assigned a node cost of 2.}
        \label{fig:motivation_1}
    \end{subfigure}%
    \hspace{0.25em}
    \begin{subfigure}[t]{0.32\linewidth}
        \centering
        \includegraphics[width=\textwidth]{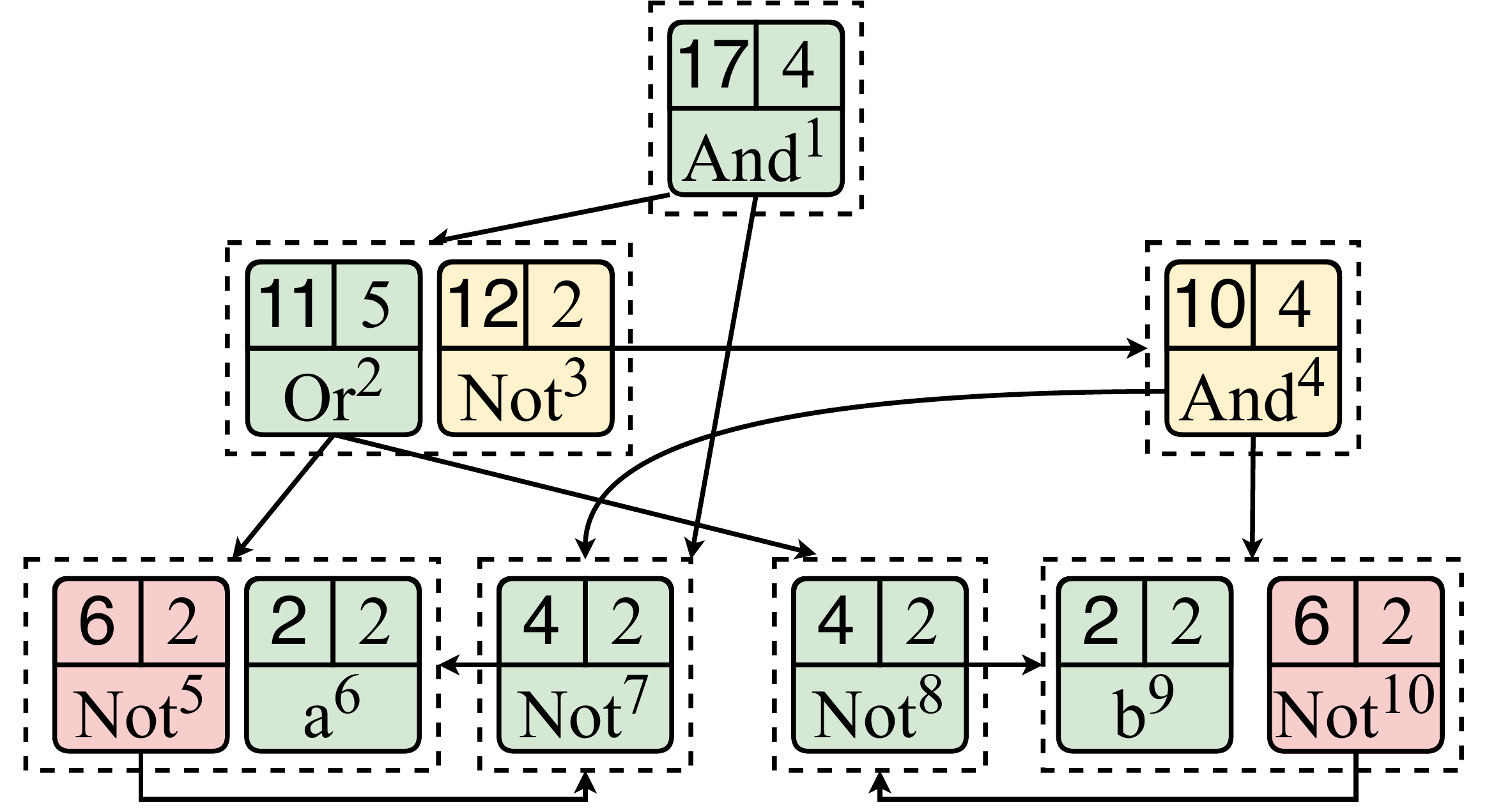}
        \caption{Heuristic e-graph extraction solution yielding a total DAG cost of 17. Green e-nodes indicate the selected nodes from the heuristic algorithm. For each e-node, the corresponding DAG cost appears in the upper left corner, capturing the sharing e-nodes.
        }
        \label{fig:motivation_2}
    \end{subfigure}%
    \hspace{0.25em}
    \begin{subfigure}[t]{0.32\linewidth}
        \centering
        \includegraphics[width=\textwidth]{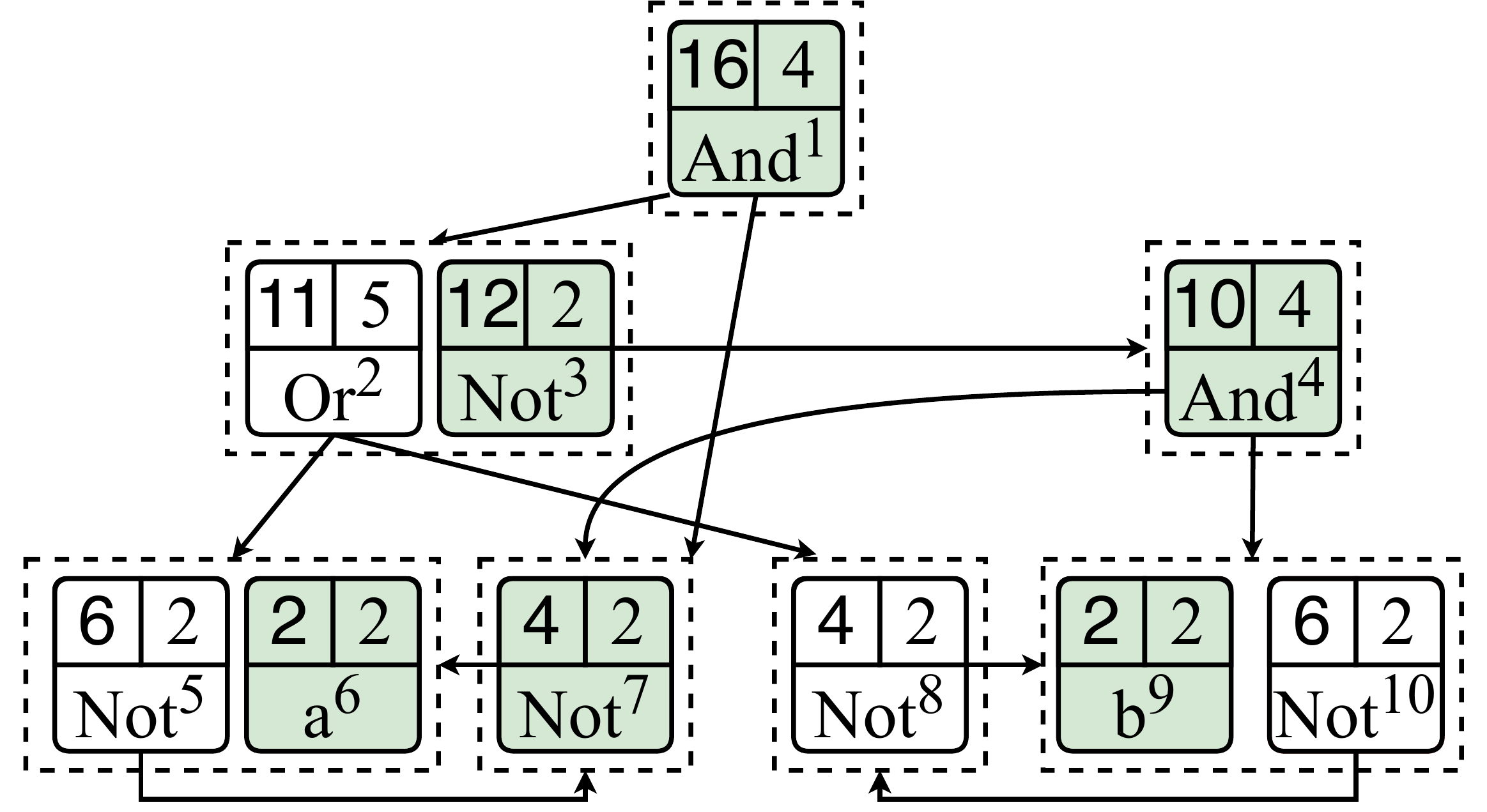}
        \caption{Final \name extraction results achieving an optimal DAG cost of 16. This improvement over the heuristic solution is enabled by the strategic sharing of e-node $\lnot a$, which allows the selection of $E_3$ despite its higher node cost compared to $E_2$.}
        \label{fig:motivation_3}
    \end{subfigure}%

    \caption{E-graph extraction example with \name. The e-graph representation for the expression $\lnot a \wedge (a \vee \lnot b)$ with two rewriting rules: double negation $(x \Leftrightarrow \lnot(\lnot x))$ and De Morgan's laws $((a \wedge b) \Leftrightarrow \lnot(\lnot a \vee \lnot b))$. Dotted boxes encapsulate e-classes, while arrows indicate parent-child relationships between e-nodes and their children e-classes. The progression from left to right demonstrates how our approach improves upon the heuristic solution to chase the optimal extraction.
    }
    \vspace{-3mm}
    \label{fig:motivations}
\end{figure*}

\textbf{Extraction overhead dominates e-graph optimization} -- In the e-graph optimization workflow, saturation expands the graph by applying equivalence rules, while extraction solves an \textit{NP-hard} combinatorial problem to select the optimal rewrite expression and thus determines final solution quality. Figure \ref{fig:prop_runtime} shows the proportional runtime for saturation (blue) and the DAG-based extraction (orange) phases across a representative set of large E-Syn\cite{chen2024syn} and BoolE\cite{yin2025boole} benchmarks in logic synthesis domain. It reveals that extraction consistently dominates end-to-end execution runtime, accounting for an average of 89.30\% of total runtime across all tested cases. These results highlight the challenges in current extraction techniques and underscore the urgent need for an efficient approach that can reduce extraction runtime without degrading optimization quality.

\textbf{Limitation} -- We use Figure ~\ref{fig:motivations} to illustrate the traditional heuristic algorithm and its limitations. Figure ~\ref{fig:motivation_1} depicts an e-graph representing the expression $\lnot a \wedge (a \vee \lnot b)$. Each e-node is assigned a cost number, and the cost value is labeled inside the node. This e-graph is saturated with two rewriting rules: double negation ($x \Leftrightarrow \lnot(\lnot x)$) and De Morgan's laws ($(a \wedge b) \Leftrightarrow \lnot(\lnot a \vee \lnot b)$). The selection of e-nodes significantly influences the total cost of the e-graph. 

The heuristic approach employs a greedy algorithm that selects the e-node with the lowest DAG cost from each e-class, recursively calculating costs from bottom to top. The DAG cost for an e-node is computed by summing the established DAG costs for its child e-classes and adding the intrinsic cost of the e-node itself. For example, $E_1$ has two child e-classes that select $E_2$ (DAG cost 11) and $E_7$ (DAG cost 4) respectively. Thus, the DAG cost of $E_1$ is $Cost_1 = 4$ (node cost of $E_1$) $+ 11$ (DAG cost of $E_2$) $+ 4$ (DAG cost of $E_7$) $ - 2$ ($E_6$ is a shared subexpression between $E_2$ and $E_7$) $= 17$. Figure ~\ref{fig:motivation_2} highlights the heuristic solution in green, yielding a total DAG cost of 17. The greedy search process is demonstrated in Algorithm \ref{alg:extraction}.

\begin{figure}[ht]
    \begin{minipage}[t]{0.48\textwidth}
        \begin{algorithm}[H]
        \caption{Greedy DAG-based E-graph Extraction}
        \label{alg:extraction}
        \scriptsize
        \begin{algorithmic}[1]
            
            
            
            
            
            
        \Procedure{CalculateCostSet}{\textit{egraph}, \textit{node\_id}, \textit{costs}}
            \State $\textit{node} \gets \textit{egraph}[\textit{node\_id}]$
            
            \If{$\textit{node}$ is leaf} 
                \State \Return $(\{\textit{node}.\textit{classid} : \textit{node}.\textit{cost}\}, \textit{node}.\textit{cost}, \textit{node\_id})$ 
            \EndIf
            
            \State $\textit{children\_classes} \gets \text{unique classes of } \textit{node}.\textit{children}$
            \State $\textit{result}[\textit{node}.\textit{classid}] \gets \text{Merge cost sets from all children classes}$
            \State $\textit{result\_cost} \gets (\text{cycle detected ? } \infty : \text{sum of all costs in } \textit{result})$ 
            
            \State \Return $(\textit{result}, \textit{result\_cost}, \textit{node\_id})$ 
        \EndProcedure
        \vspace{\baselineskip} 
    
        \Procedure{Extract}{\textit{egraph}}
            \State $\textit{pending} \gets \textit{ENodes}(G_e)$ \Comment{Initializing the Queue}
            \State $\textit{C\_Costs} \gets \{\}$ \Comment{Mapping ClassId to selected CostSet}
            \State $\textit{N\_Costs} \gets \{\}$  \Comment{Tracking node cost for space pruning}
            
            \While{$\textit{node\_id} \gets \textit{pending}.\textit{pop}()$}
                \If{$\forall c \in \textit{egraph}[\textit{node\_id}].\textit{children}() : c \in \textit{C\_Costs}$} 
                    \State $\textit{cost\_set} \gets \textit{CalculateCostSet}(\textit{egraph}, \textit{node\_id}, \textit{C\_Costs})$
                    \If{$\textit{cost\_set}.\textit{second} < \text{previous best cost in } \textit{C\_Costs}$}
                        \State Update $\textit{C\_Costs}$ and insert $\textit{cost\_set}$ for $\textit{egraph}[\textit{node\_id}].\textit{class}$ 
                    \EndIf
                    \If{$\textit{cost\_set}.\textit{second} < \text{previous best cost in } \textit{N\_Costs}$}
                         \State Update $\textit{N\_Costs}.\textit{second}$ and insert $\textit{cost\_set}$ for $\textit{node\_id}$ 
                    \EndIf
                \EndIf
            \EndWhile
            
            \State \Return extraction result based on optimal choices
        \EndProcedure
        \end{algorithmic}
        \end{algorithm}
    \end{minipage}
    \hfill
    \begin{minipage}[t]{0.48\textwidth}
        \begin{algorithm}[H]
        \caption{{Parallelized Extraction} 
        }
        \label{alg:paralleled_extraction}
        \scriptsize
        \begin{algorithmic}[1]

    \Procedure{Extract\_Parallelized}{\textit{egraph}}
        \State Initialize $\textit{pending}, \textit{C\_Costs}$ and $\textit{N\_Costs}$
        
        \While{$\textit{vec\_node\_id} \gets \textit{pending}.\textit{pop\_batch}()$\Comment{Pop a batch of elements}}
            \State $\textit{Inserted} \gets \{\}$ 
            \For {$\textit{node\_id} \in \textit{vec\_node\_id}$} \Comment{Parallel computation for all elements}
                \If{$\forall c \in \textit{egraph[node\_id]}.\textit{children}() : c \in \textit{C\_Costs}$}
                    \State $\textit{cost\_set} \gets \textit{CalculateCostSet}(\textit{egraph}, \textit{node\_id}, \textit{C\_Costs})$
                    \If{$\textit{cost\_set}.\textit{second} < \text{previous best cost in } \textit{C\_Costs}$}
                        \State $\textit{Inserted}.\textit{push\_back}(\textit{cost\_set})$
                    \EndIf
                    \State Update $\textit{N\_Costs}$ if a better cost is found
                \EndIf
            \EndFor
            \State Deduplicate $\textit{Inserted}$ to retain only entries with minimum cost
            \State Insert all items from $\textit{Inserted}$ into $\textit{C\_Costs}$
        \EndWhile
        \State \Return extraction result based on optimal choices
    \EndProcedure
        \end{algorithmic}
        \end{algorithm}
    \end{minipage}
    \vspace{-5mm}
\end{figure}

\textbf{Key Observations} -- Through analysis, we derive key observations:

\begin{enumerate}[leftmargin=*]
    \item \textbf{Local optima do not guarantee global optimality:} Selecting the e-node with the minimum local DAG cost within an e-class does not necessarily lead to the global optimal overall solution. For instance, in the e-class containing nodes $E_2$ and $E_3$, $E_2$ has a lower local DAG cost (11) compared to $E_3$ (12). However, considering the entire expression, choosing $E_3$ actually yields a superior global solution due to the shared sub-expression $E_7$, as demonstrated in the green node of Figure~\ref{fig:motivation_3}. Selecting $E_3$ in this case leads to a total DAG cost of 16 for the root expression, calculated as $4$ (\text{node cost of } $E_1$) $+ 12$ (\text{DAG cost of } $E_3$) $+ 4$ (\text{DAG cost of } $E_7$) $ - 4$ (\text{accounting for } $E_7$ \text{ being a shared subexpression}) $= 16$. This example illustrates why purely heuristic approaches can result in suboptimal final solutions.
    \item \textbf{{Local minima} provide effective search space pruning:} While greedy selection is not globally optimal, there is often a strong correlation: e-nodes exhibiting lower local DAG costs tend to be included in the globally optimal solution more frequently. Exploiting this correlation enables substantial search space pruning by focusing on these high-potential candidates. This problem can be reduced sufficiently to allow the application of exact solvers (such as ILP) on the pruned space. {We will provide more details in Section \ref{sec:exactsolving} and Section \ref{sec:prunespace}.} 

    \item \textbf{Extraction exhibits weak data dependence:} The convergence of the heuristic DAG cost calculation does not strictly depend on the order in which e-nodes are processed. While different processing orders might lead to different intermediate cost values, the algorithm ultimately converges to a stable state where locally optimal e-nodes are selected within each e-class. This property is significant because it enables the parallel computation of DAG costs for multiple e-nodes simultaneously without compromising the correctness or the quality of the final converged heuristic solution. {More explanations will be presented in Section \ref{sec:multi}.} 
    
\end{enumerate}


\section{Approach}


\label{sec:methodology}

In this section, we present a comprehensive overview of \name, our novel framework for e-graph extraction. \name accepts a saturated e-graph as input and produces a high-quality selection of e-nodes from each e-class, substantially improving extraction results while maintaining computational efficiency. 
\texttt{e-boost} consists of four key components: (1) a parallelized heuristic e-graph extraction algorithm that efficiently computes preliminary DAG costs, (2) a search space pruning technique that leverages these heuristic results to significantly reduce the problem size, (3) 
{exact solving with pruning and warm-start}, and (4) domain-specific optimizations that accelerate the extraction process. We explore each of these components in detail throughout this section.

\subsection{Multi-threading E-Graph Extraction}

\label{sec:multi}

Our approach employs heuristic algorithms for efficient search space pruning and provides a high-quality warm start for the subsequent exact solver. Building on established extraction methods, Algorithm \ref{alg:extraction} implements an efficient bottom-up approach for e-graph extraction. The algorithm first identifies all parent-child relationships within the e-graph and initializes analysis from leaf nodes. The function \textsc{CalculateCostSet} computes the cost of selecting a specific e-node while accounting for shared subexpressions to eliminate redundant calculations. Then it merges cost sets and detects potential cycles that would lead to invalid expressions. For core function \textsc{Extract}, two key data structures guide the extraction process: \texttt{C\_Costs} maintains the optimal selection for each e-class along with its associated cost, forming the foundation for the final extraction result. Additionally, \texttt{N\_Costs} tracks individual e-node costs, enabling efficient search space pruning by quickly identifying suboptimal nodes. This cost tracking system counts shared nodes only once, accurately representing expressions with repeated substructures. The extraction process proceeds through dynamic programming, iteratively evaluating node cost contributions and updating the global solution whenever a more efficient representation is discovered. Cost optimizations propagate upward through the e-graph until reaching a stable state, where each e-class has selected the locally optimal e-node from all candidates, yielding a converged extraction solution. Figure \ref{fig:motivation_2} illustrates a heuristic extraction example, with green e-nodes highlighting the selected extraction results.

\textbf{Parallelized Extraction} In Algorithm \ref{alg:extraction}, the while loop from lines 15-25 constitutes the runtime bottleneck, dominating total execution time. 
To accelerate the convergence runtime, we developed a parallelized optimized version for the algorithm. The optimized multi-threading extraction algorithm is presented in Algorithm \ref{alg:paralleled_extraction}. Leveraging the weak data dependence property identified in our motivating example, this algorithm implements fine-grained parallelism by processing multiple e-nodes simultaneously in each iteration, with the batch size dynamically set according to the available thread count. Rather than serially updating the shared \texttt{C\_Costs} data structure, the algorithm employs a temporary \texttt{Inserted} collection to minimize lock contention and reduce thread synchronization overhead. For each batch of e-nodes, the algorithm performs parallel cost calculations, collects potential updates in the \texttt{Inserted} container, and then applies a deduplication step (line 14) that ensures only entries with minimum costs are retained before updating the global state. This approach significantly reduces the number of while loop iterations required for convergence, as each thread can compute costs independently without affecting algorithm correctness. The deduplication mechanism guarantees that concurrent updates to the same e-class maintain the invariant that only the minimum cost representation is preserved, thereby ensuring solution quality while achieving substantial performance improvements. 

It is important to note the data dependencies inherent in this parallel approach. Specifically, intermediate states during convergence, and potentially the exact final solution, may differ from sequential execution or between different parallel runs. However, this dependence is weak: the algorithm is guaranteed to converge to a stable state where each e-class selects a local optimum. We now proceed with a formal analysis of this convergence behavior and the implications of this weak data dependence:

\begin{theorem}
The parallelized extraction algorithm converges to a solution with equivalent quality to the sequential algorithm despite weak data dependence in execution order.
\end{theorem}

\begin{proof}
Let $G$ be the e-graph and $V$ be the set of e-nodes collection. Define $S_t$ as the state of \texttt{C\_Costs} at iteration $t$, where each e-class $c$ is mapped to its currently selected e-node and associated cost. Let $S^*$ represent the stable state where no further updates to \texttt{C\_Costs} occur.

Let $S[v] \subseteq V$ be the set of unique e-nodes in the minimum-cost children e-classes at e-node $v$. The cost for any e-node $v \in V$ is:
\begin{equation} \label{eq:dag_cost_formal_no_mathrm}
Cost(v) = \sum_{n \in S[v] \cup \{v\}} \mathrm{node\_cost}(n)
\end{equation}
where $node\_cost(n)$ is the intrinsic node cost associated with the individual e-node $n$.

The sequential algorithm processes one node per iteration, updating \texttt{C\_Costs} when a better solution is found:
\begin{equation}
\forall c \in \text{e-classes}(G), S_t[c] \geq \min_{v \in c} cost(v)
\end{equation}

In the parallelized algorithm, multiple e-nodes are processed concurrently before updating \texttt{C\_Costs}. This changes the order of executions but preserves the essential mechanisms:

1. The condition in line 8 of Algorithm \ref{alg:paralleled_extraction} ensures that an e-node cost is only calculated when all its children e-classes have stable selections, enforcing bottom-up execution order.

2. The deduplication step (line 14) guarantees that for each e-class, only the minimum cost selection is retained:
\begin{equation}
\forall c \in \text{e-classes}(G), S_{t+1}[c] = \min(S_t[c], \min_{v \in candidates_t(c)} cost(v))
\end{equation}
where $candidates_t(c)$ represents e-nodes in class $c$ processed in the current batch.

Both algorithms terminate when no further improvements to \texttt{C\_Costs} are possible. While intermediate states $S_t$ differ between sequential and parallel executions due to evaluation order, both algorithms converge to a similar fixed point $S^*$ characterized by:

\begin{equation}
\forall c \in \text{e-classes}(G), cost(S^*[c]) = \min_{v \in c} cost(v)
\end{equation}

The fixed points may not be identical when multiple e-nodes within an e-class have equal minimum costs, as the algorithm preserves the first encountered minimum-cost solution. However, any stable state reached by the algorithm is equally valid for the purposes of DAG cost extraction. The objective is not to pursue a specific fixed point, but rather any stable state, as all such states provide equivalent heuristic cost information for subsequent processing. Overall, the weak data dependence means that while intermediate states $S_t$ may differ between sequential and parallel executions, both algorithms converge to a similar fixed point $S^*$. This fixed point is uniquely determined by the structure of the e-graph and node costs, independent of execution order. Therefore, the parallelized algorithm produces a solution with equivalent quality to the sequential algorithm, despite differences in the traversal path to convergence.

\end{proof}

\subsection{Adaptive Search Space Reduction}
\label{sec:prunespace}

We leverage the cost information recorded in \texttt{N\_Costs} from Algorithm \ref{alg:paralleled_extraction} to strategically prune the search space for exact solving. For each e-class, we identify the minimum DAG cost ($cost_{min}$) among all candidate e-nodes and apply a parameterized threshold $\theta$ ($\theta \geq 1$) to determine which nodes to retain. Specifically, we preserve only those e-nodes with costs less than or equal to $cost_{min} \cdot \theta$, effectively filtering out suboptimal candidates while maintaining promising alternatives.

Figure \ref{fig:motivation_2} illustrates this pruning mechanism with $\theta = 1.25$. In the e-class containing $E_2$ and $E_3$, $E_2$ has the minimum DAG cost of 11, while $E_3$ has a cost of 12. Since 12 < $11 \cdot 1.25 = 13.75$, both e-nodes are retained in the search space. Conversely, in the e-class containing $E_5$ and $E_6$, $E_6$ has a minimum cost of 2, while $E_5$ has a cost of 6. Since 6 > $2 \cdot 1.25 = 2.5$, $E_5$ is pruned from the search space, and only $E_6$ is preserved for exact solving. Similarly, $E_9$ is retained in the search space, while $E_{10}$ is pruned because its cost exceeds the threshold relative to the minimum cost in its e-class.

Building on the pruning approach, we acknowledge that the global optimality cannot be guaranteed in all cases. In an extreme scenario where the globally optimal solution includes e-nodes with costs significantly exceeding the minimum within their respective e-classes, these optimal candidates will be pruned by our threshold mechanism. However, across our extensive benchmark testing, we rarely encountered such a scenario. This empirical observation validates our approach as an effective balance between theoretical optimality and practical performance, enabling near-optimal solutions for problem instances.

\subsection{{Exact solving with pruning and warm-start} 
}

\label{sec:exactsolving}

We formulate the e-graph extraction problem with search space pruning as an Integer Linear Program based on the approach in \cite{yang2021equality}. The objective and constraints are formulated in Equation \ref{eq:ilp_formulation}, where $s_i$ represents the selection of e-node $i$, $c_i$ is its cost, $A_j$ indicates activation of e-class $j$, and $\mathcal{C}$, $\mathcal{N}$, and $\mathcal{R}$ denote the sets of e-classes, e-nodes, and root e-classes, respectively. Variables $L_j$ establish level assignments to prevent cycles, with $\text{Opp}_i$ serving as opposite variables in the cycle prevention constraints. $M$ is a sufficiently large constant set to $|\mathcal{C}|+1$. Variable $s_n^{heu}$ is the selection from heuristic algorithm solution. Constraint (5g) and (5h) represent our contribution, where $\mathcal{P}$ is the set of pruned e-nodes identified by our heuristic algorithm. This constraint explicitly removes low-potential candidates from consideration, substantially reducing the solver's search space while maintaining solution quality. Additionally, we leverage the heuristic solution as a warm-start initialization for the exact solver, providing it with a high-quality starting point. Specifically, we pass the extracted nodes from our heuristic algorithm (illustrated as green nodes in Figure \ref{fig:motivation_2}) to guide the exact solving process.
\begin{align}
\small
\min_{s \in \{0,1\}} f(s) &= \sum_{i=1}^{N} c_i s_i \tag{5a} \\
\text{s.t.} \sum_{n_i \in C_j} s_i &= A_j \quad \forall j \in \mathcal{C} \tag{5b} \\
s_i &\leq A_j \quad \forall i, \forall j \in \text{children}(i) \tag{5c} \\
A_r &= 1 \quad \forall r \in \mathcal{R} \tag{5d} \\
L_j - L_k + M \cdot \text{Opp}_i &\geq 1 \quad \forall i, \forall k \in \text{children}(i), j = \text{class}(i), j \neq k \tag{5e} \\
s_i + \text{Opp}_i &= 1 \quad \forall i \in \mathcal{N} \tag{5f} \\
s_i &= 0 \quad \forall i \in \mathcal{P} \tag{5g} \\
s_n \leftarrow s_n^{heu} & \tag{5h} \quad \forall n \in \mathcal{N}  \\
s_i, A_j,& \text{Opp}_i \in \{0,1\} \tag{5i} \\ \label{eq:ilp_formulation}
\end{align}

\subsection{Optimization for \name}

In addition to our algorithmic improvements, we implement several key optimizations in \name:

\begin{itemize}
    \item We apply redundancy elimination by retaining only one e-node with minimum cost when multiple e-nodes share identical child e-classes within an e-class. This optimization is safe since the optimal solution always selects the minimum-cost option. For example, in Boolean algebra e-graphs with commutative laws, expressions like $(+ \thickspace a \thickspace b)$ and $(+ \thickspace b \thickspace a)$ co-exist in the same e-class. Assuming \textit{addition} has the lowest cost among operators with $a$ and $b$ as children, we preserve just one addition variant.
    \item We implement memory-efficient data structures with adaptive typing for hash maps. For e-graphs under 65,536 e-classes, we use 16-bit unsigned integers as hash keys, reducing memory by 50\% compared to 32-bit integers. For larger e-graphs, we automatically scale to 32-bit integers, optimizing memory usage while maintaining performance across diverse problem sizes.
    \item While our research primarily focuses on DAG cost extraction, we also extend Algorithm \ref{alg:paralleled_extraction} to handle tree and depth cost extraction.
\end{itemize}

\section{Experiment}

\begin{figure*}[htbp]
    \centering
    \centering
    \begin{subfigure}[t]{0.35\linewidth}
        \hfill
        \includegraphics[width=\textwidth]{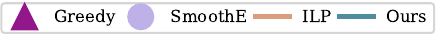}
        \label{fig:legend}
    \end{subfigure}
    \\
    \vspace{-2mm}
    \hfill
    \begin{subfigure}[t]{0.33\linewidth}
        \centering
        \includegraphics[width=\textwidth]{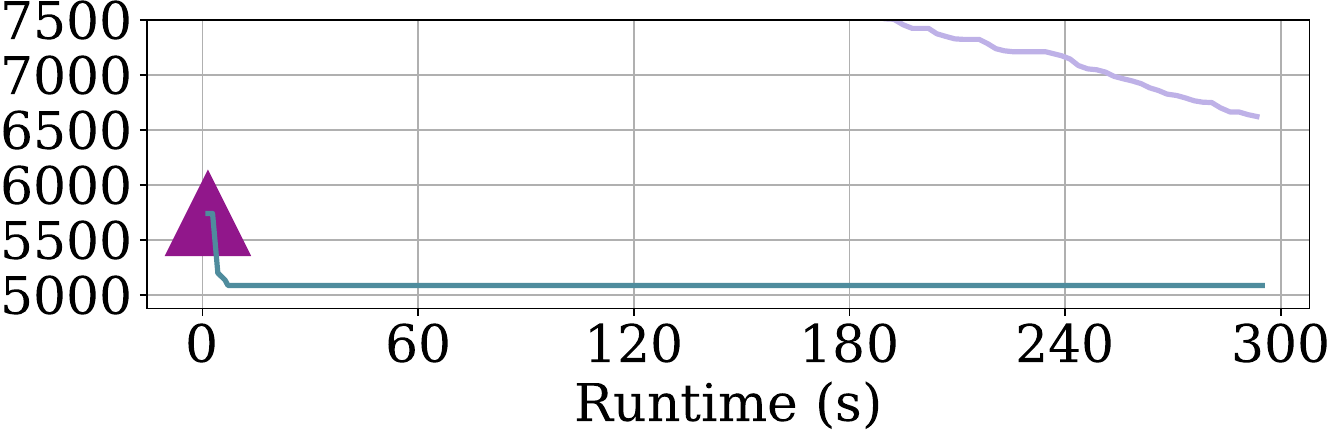}
        \caption{BoolE, mul32, CPLEX. ILP timeout}
        \label{fig:BoolE_mul32_Cplex}
    \end{subfigure}%
    \hfill
    \begin{subfigure}[t]{0.33\linewidth}
        \centering
        \includegraphics[width=\textwidth]{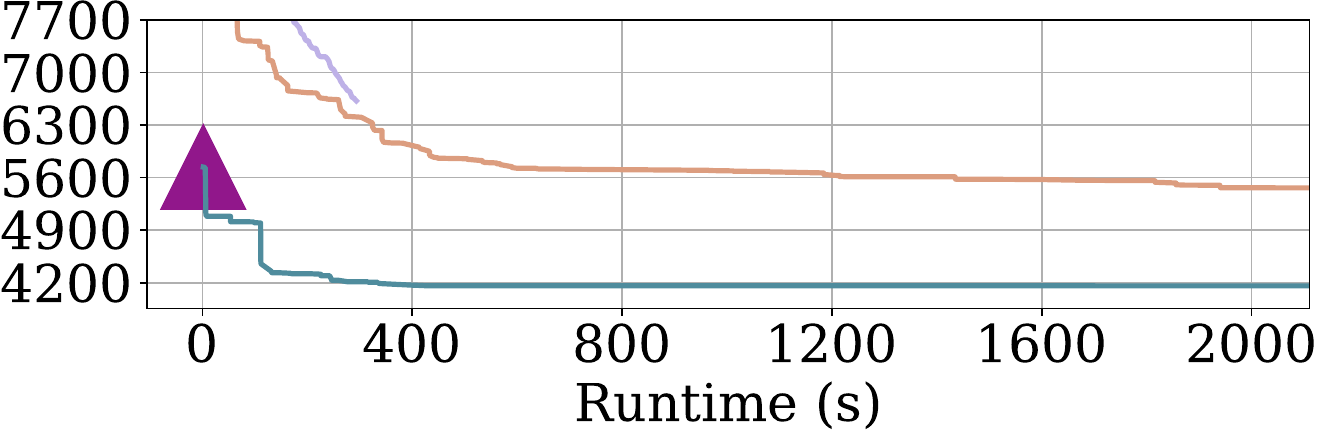}
        \caption{BoolE, mul32, Gurobi}
        \label{fig:BoolE_mul32_Gurobi}
    \end{subfigure}
    \hfill
    \begin{subfigure}[t]{0.33\linewidth}
        \centering
        \includegraphics[width=\textwidth]{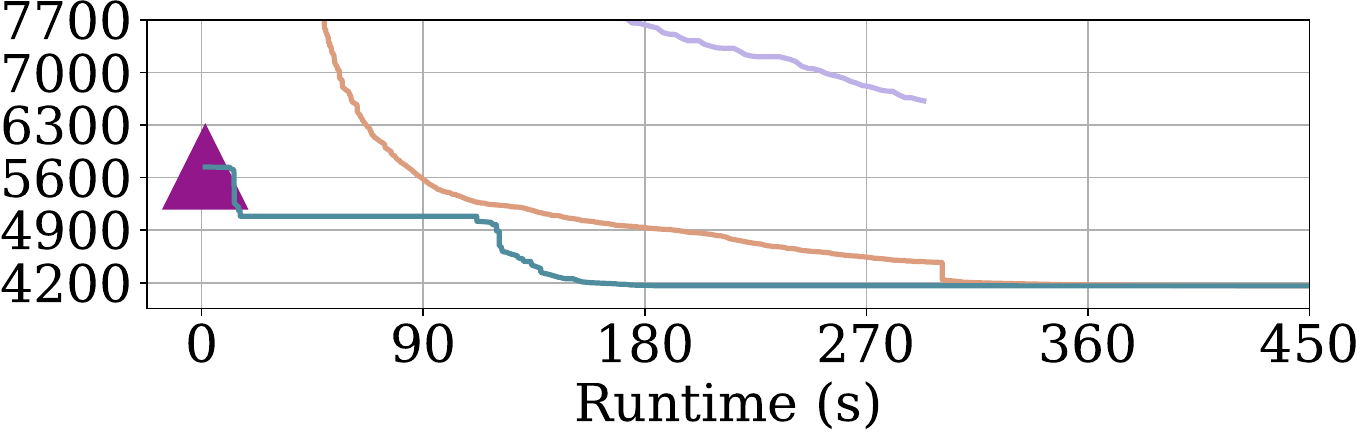}
        \caption{BoolE, mul32, CP-SAT}
        \label{fig:BoolE_mul32_CP-SAT}
    \end{subfigure}
    \hfill
    \begin{subfigure}[t]{0.33\linewidth}
        \centering
        \includegraphics[width=\textwidth]{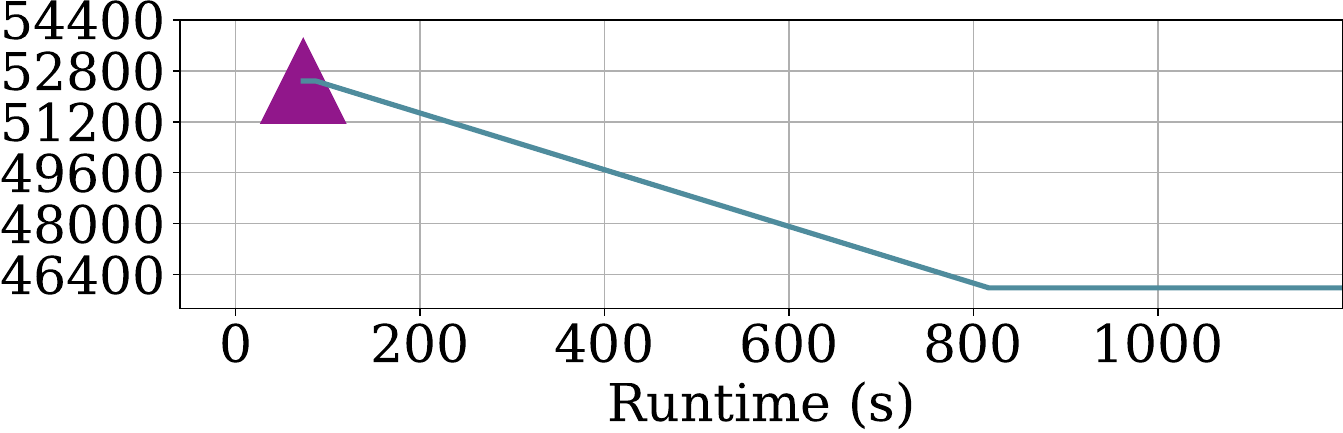}
       \captionsetup{justification=centering}
        \caption{BoolE, mul96. ILP, SmoothE failed}
        \label{fig:BoolE_mul96_Cplex}
    \end{subfigure}%
    \hfill
    \begin{subfigure}[t]{0.33\linewidth}
        \centering
        \includegraphics[width=\textwidth]{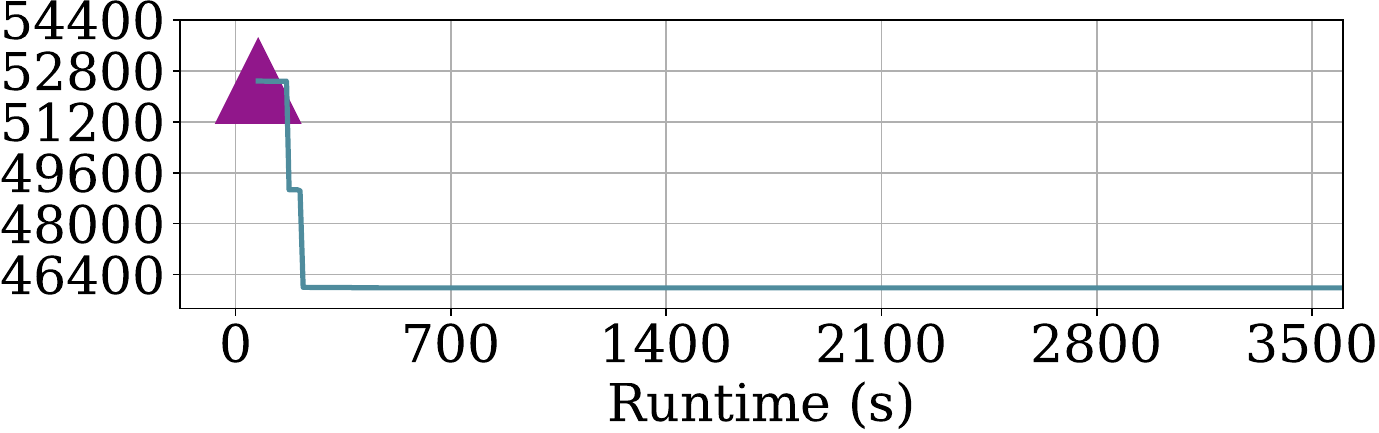}
        \captionsetup{justification=centering}
        \caption{BoolE, mul96. Gurobi, ILP, SmoothE failed}
        \label{fig:BoolE_mul96_Gurobi}
    \end{subfigure}
    \hfill
    \begin{subfigure}[t]{0.33\linewidth}
        \centering
        \includegraphics[width=\textwidth]{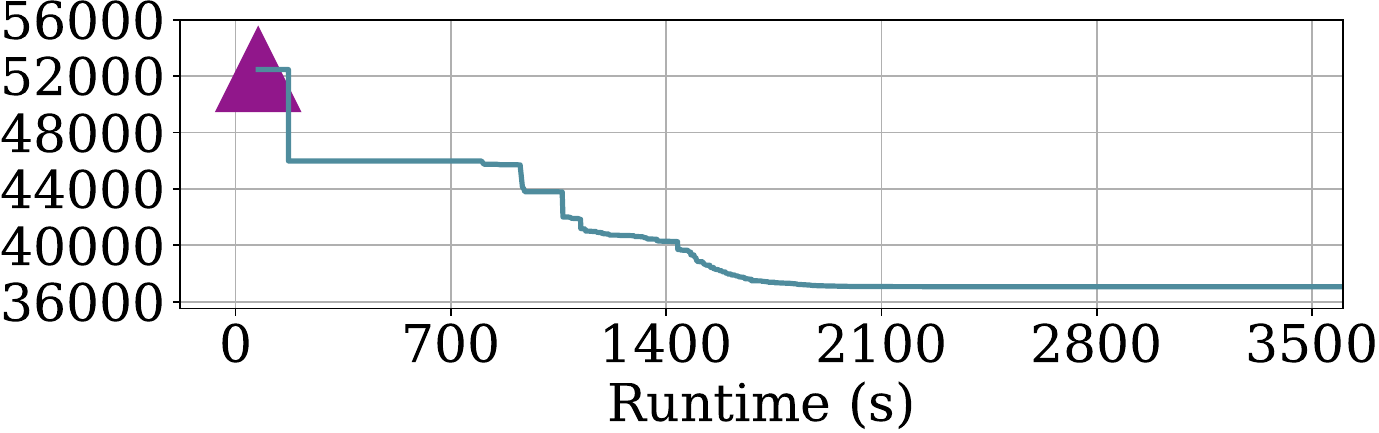}
        \captionsetup{justification=centering}
        \caption{BoolE, mul96. CP-SAT, ILP, SmoothE failed}
        \label{fig:BoolE_mul96_CP-SAT}
    \end{subfigure}
    \hfill
    \begin{subfigure}[t]{0.33\linewidth}
        \centering
        \includegraphics[width=\textwidth]{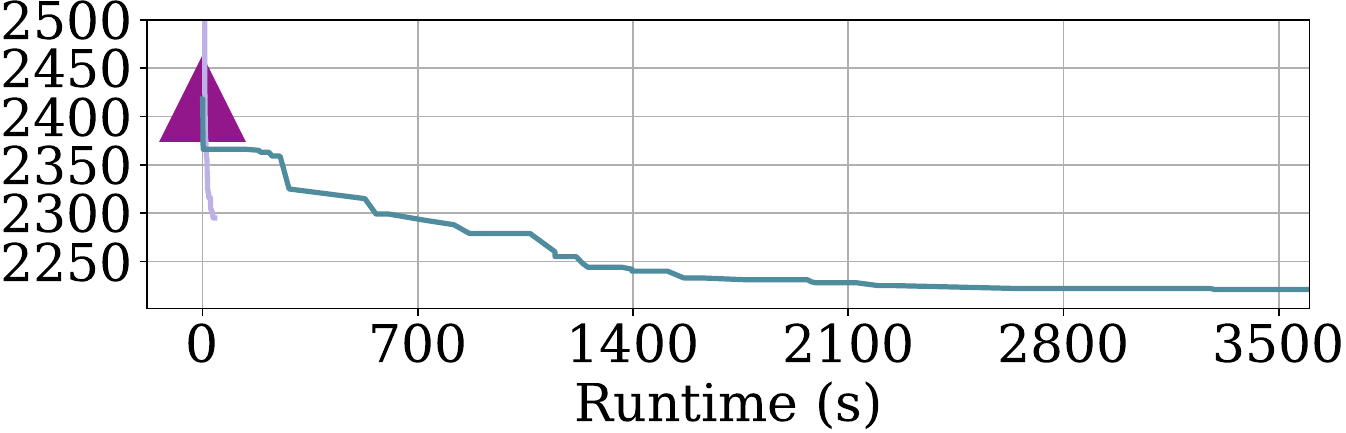}
        \caption{ESyn, c2670, CPLEX, ILP timeout}
        \label{fig:esyn_c2670_CPLEX}
    \end{subfigure}%
    \hfill
    \begin{subfigure}[t]{0.33\linewidth}
        \centering
        \includegraphics[width=\textwidth]{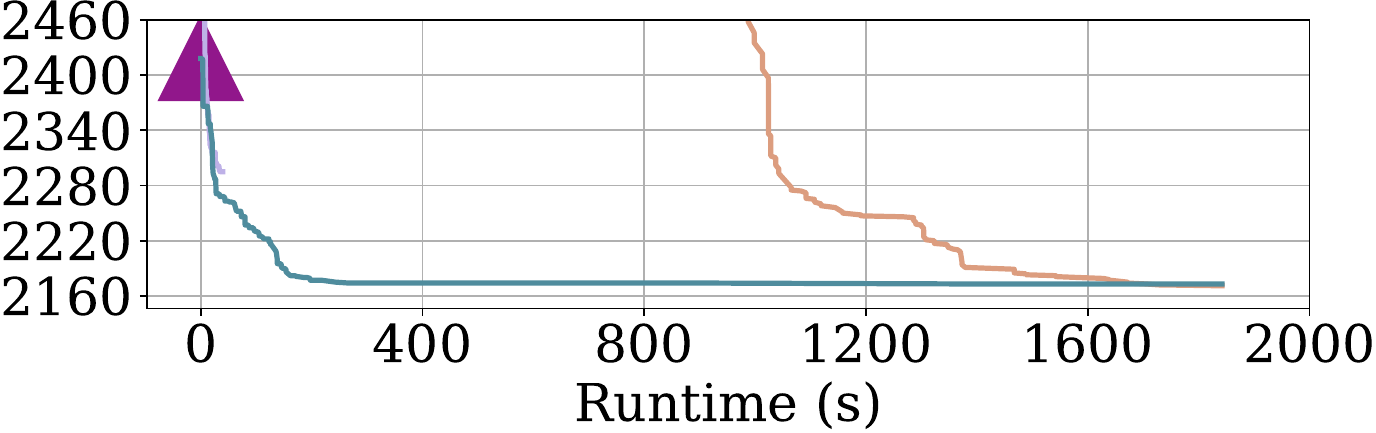}
        \caption{ESyn, c2670, Gurobi}
        \label{fig:esyn_c2670_Gurobi}
    \end{subfigure}
    \hfill
    \begin{subfigure}[t]{0.33\linewidth}
        \centering
        \includegraphics[width=\textwidth]{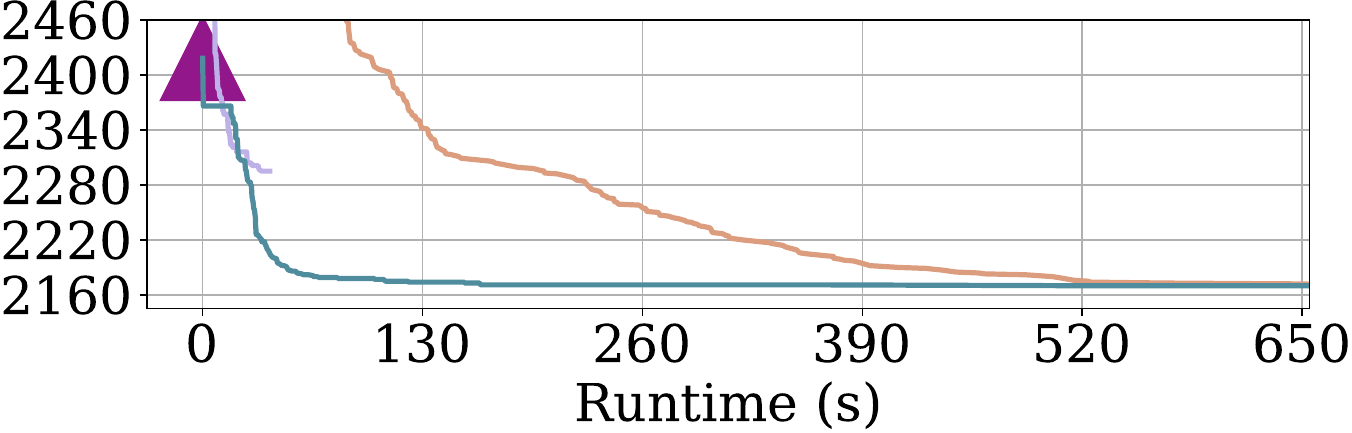}
        \caption{ESyn, c2670, CP-SAT}
        \label{fig:esyn_c2670_CP-SAT}
    \end{subfigure}
    \hfill
    \begin{subfigure}[t]{0.33\linewidth}
        \centering
        \includegraphics[width=\textwidth]{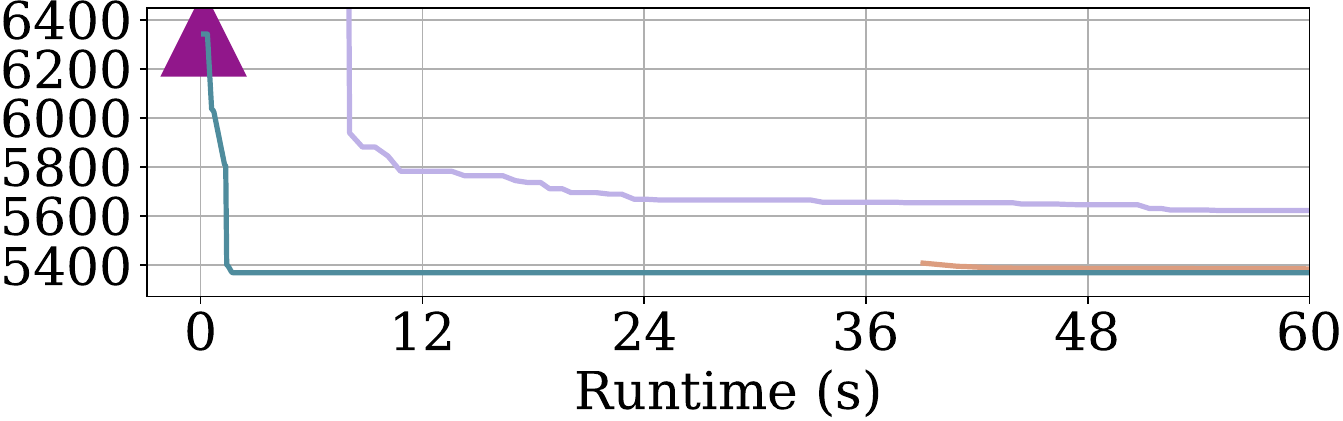}
        \captionsetup{justification=centering}
        \caption{Esyn, qdiv, CPLEX \\ \name and ILP overlapped}
        \label{fig:esyn_qdiv_cplex}
    \end{subfigure}%
    \hfill
    \begin{subfigure}[t]{0.33\linewidth}
        \centering
        \includegraphics[width=\textwidth]{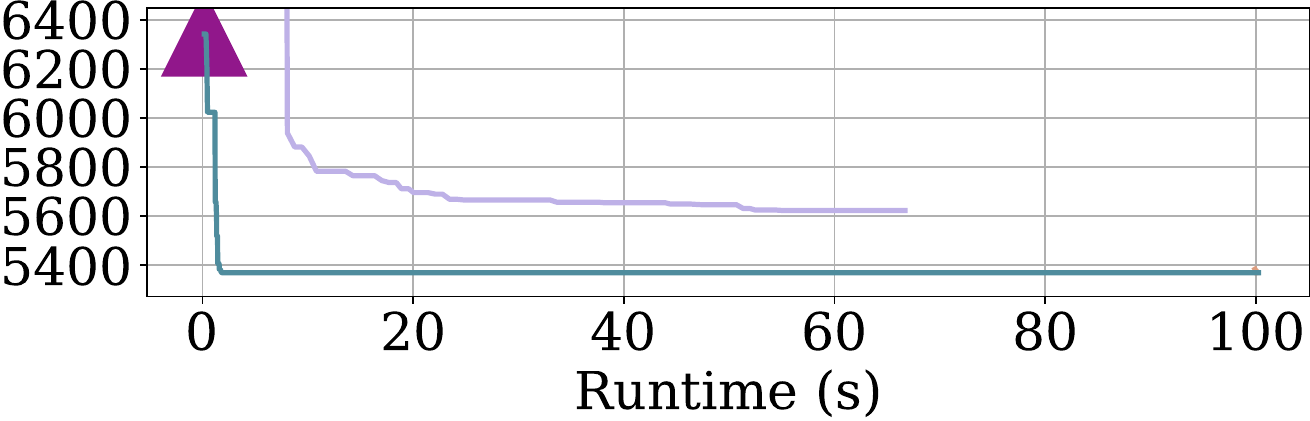}
        \captionsetup{justification=centering}
        \caption{Esyn, qdiv, Gurobi \\ \name and ILP overlapped}
        \label{fig:esyn_qdiv_gurobi}
    \end{subfigure}
    \hfill
    \begin{subfigure}[t]{0.33\linewidth}
        \centering
        \includegraphics[width=\textwidth]{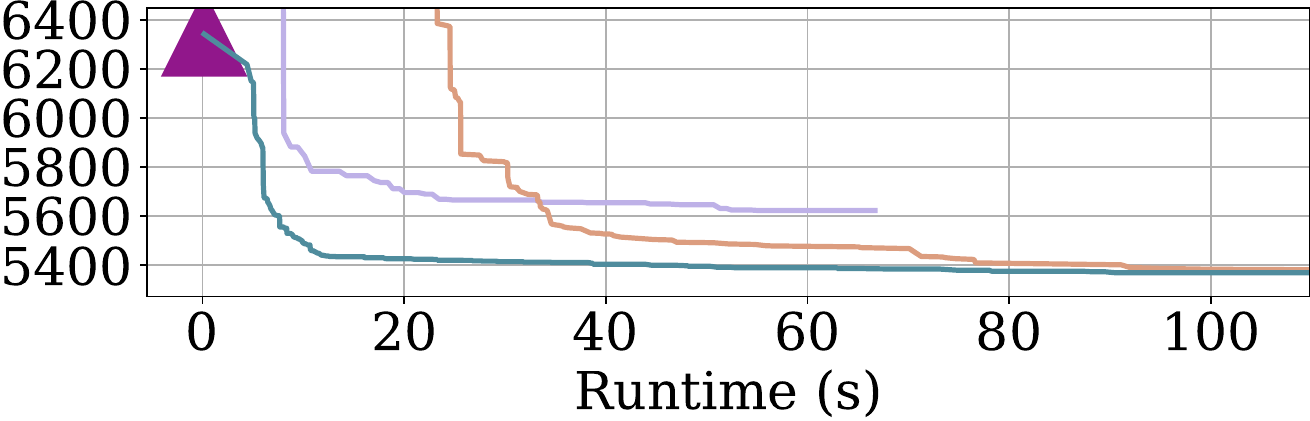}
        \caption{Esyn, qdiv, CP-SAT}
        \label{fig:esyn_qdiv_cpsat}
    \end{subfigure}
    \vspace{-1mm}
    \caption{Performance comparison between \name, SmoothE, and vanilla exact solving for representative benchmarks. For fair comparison, all curves requiring heuristic initialization are right-shifted by the heuristic runtime. Overall, \name delivers 19.04\% performance improvement over SmoothE while providing $138\times$ runtime speedup to achieve equivalent results. Compared to vanilla exact solving (with a 3600s timeout), our approach achieves 5.6\% better performance quality while providing $558\times$ runtime speedup to reach equivalent-quality solutions.}
    \label{fig:extract_results}
    \vspace{-5mm}
\end{figure*}

Our experiments were conducted on a server with an Intel Xeon Gold 6418H CPU (48 physical cores, 96 logical threads) and 1TB RAM. For GPU-based comparisons, we utilized two NVIDIA A6000 GPUs with 96 GB total memory. 
We evaluated \name with diverse benchmarks from: (1) logic synthesis benchmarks from BoolE\cite{yin2025boole}, E-Syn \cite{chen2024syn,chen2025emorphicscalableequalitysaturation} suites; and (2) computational kernels from program optimization frameworks including IMpress\cite{ustun2022impress}, TENSAT\cite{yang2021equality}, ROVER \cite{coward2023automating,coward2022automatic}, and Diospyros\cite{vanhattum2021vectorization}.
We compared \name against two SOTA extraction approaches: (1) vanilla ILP solving; and (2) SmoothE \cite{cai2025smoothe}, the differentiable e-graph extraction framework.
All implementations use the same e-graph infrastructure, and cost models were adjusted to non-negative integers for fair comparison.

To evaluate the effectiveness of \name, we will discuss the following research questions (RQs) in this section.

\subsection{Performance of E-graph extraction}

\noindent
\textbf{RQ1: How effective and scalable is \name for e-graph extraction?} We test the performance of various extraction methods on benchmarks with a 3600s timeout, comparing SmoothE, vanilla ILP, and \name. Figure~\ref{fig:extract_results} displays the convergence curves, with runtime on the X-axis and objective costs on the Y-axis. For fair comparison, all curves of \name are right-shifted by the corresponding heuristic runtime. For benchmarks such as mul32 (Figure \ref{fig:BoolE_mul32_Cplex} - Figure \ref{fig:BoolE_mul32_CP-SAT}), \name consistently outperforms other approaches, converging faster to lower-cost solutions. In larger instances such as mul96 (Figure \ref{fig:BoolE_mul96_Cplex} - Figure \ref{fig:BoolE_mul96_CP-SAT}), all methods except \name fail to provide a valid solution. For the c2670 and qdiv benchmarks (Figure \ref{fig:esyn_c2670_CPLEX} - Figure \ref{fig:esyn_qdiv_cpsat}), \name rapidly converges to near-optimal solutions. 


Additionally, we present comprehensive runtime performance data in Table~\ref{tab:comprehensive_convergence}, focusing on large e-graphs exceeding 10,000 e-nodes. 
This table compares the convergence speed of our approach (\name) using different exact solvers (CPLEX \cite{ibm_ilog_cplex_optimization_studio_2025}, Gurobi \cite{gurobi_2025}, CP-SAT\cite{google_cp_sat_2024}) and pruning thresholds ($\theta \in \{1, 1.05, 1.25, 1.5\}$).

Performance is measured by the time required to reach specific normalized optimality gaps, denoted by $\alpha$. 
This metric is calculated as $\alpha = (\texttt{Current\_Cost} - \texttt{BKS}) / (H - \texttt{BKS})$, where $\texttt{BKS}$ represents the best-known solution cost for the benchmark and $H$ is the initial heuristic solution cost. 
Lower $\alpha$ values indicate better solutions, with $\alpha=0$ signifying that the best-known solution has been reached. 
The table reports runtime for achieving target $\alpha$ values of $\{0.5, 0.25, 0.1, 0.05, 0.0\}$. 
For example, for benchmark \texttt{mul32}, \name reaches optimal ($\alpha = 0$) in 200.34 seconds (marked in red) using CP-SAT with $\theta = 1.25$. 
For reference, we also include SmoothE's performance, reporting its runtime and the final $\alpha$ value it achieved within the time limit. 
Note that an achieved $\alpha > 1$ indicates a solution worse than the initial heuristic cost $H$.

\begin{table*}[htbp]
\scriptsize
\centering
\caption{\textbf{Comprehensive Runtime Performance Comparison.} This table compares runtimes (seconds) of our approach (\name) using exact solvers (CPLEX, Gurobi, CP-SAT) with varying pruning thresholds ($\theta \in \{1, 1.05, 1.25, 1.5\}$) against the SmoothE baseline on benchmarks >10k e-nodes. A special case, denoted $\theta = \text{ILP}$, represents vanilla ILP solving without heuristic pruning or warm-start. Runtimes are reported for achieving specific normalized optimality gaps $\alpha = (\texttt{Current\_Cost} - \texttt{BKS}) / (H - \texttt{BKS})$, where BKS is the best-known solution cost, $H$ is an initial heuristic cost, and lower $\alpha$ is better ($\alpha=0$ reaches BKS). Target $\alpha$ values are $\{0.5, 0.25, 0.1, 0.05, 0.0\}$. For example, for benchmark \texttt{mul32}, \name reaches optimal ($\alpha = 0$) in 200.34 seconds using CP-SAT with $\theta = 1.25$ (marked in red). Benchmark e-node counts and SmoothE performance (runtime, achieved $\alpha$) are included for reference. The symbol $\infty$ indicates failure to reach the target $\alpha$ within the time limit, and OOM denotes cases where the solver ran out of memory.}
\resizebox{\textwidth}{!}{
\begin{tabular}{l|c|c|c|c|c|c|c|c|c|c|c|c|c|c|c|c|c|c|c|c|c}
\hline
\multirow{2}{*}{\textbf{Benchmark} } & \multirow{2}{*}{\textbf{\#e-nodes} } & \multicolumn{2}{c|}{\textbf{SmoothE}} & \multirow{2}{*}{\textbf{{$\theta$}}} & \multicolumn{5}{c|}{\textbf{CPLEX-$\alpha$}} & \multirow{2}{*}{\textbf{{$\theta$}}} & \multicolumn{5}{c|}{\textbf{Gurobi-$\alpha$}} & \multirow{2}{*}{\textbf{{$\theta$}}} & \multicolumn{5}{c}{\textbf{CP-SAT-$\alpha$}} \\
\cline{3-4} \cline{6-10} \cline{12-16} \cline{18-22}
& & \textbf{Time}  & \textbf{$\alpha$} & & \textbf{0.5} & \textbf{0.25} & \textbf{0.1} & \textbf{0.05} & \textbf{0.0} &  & \textbf{0.5} & \textbf{0.25} & \textbf{0.1} & \textbf{0.05} & \textbf{0.0} &  & \textbf{0.5} & \textbf{0.25} & \textbf{0.1} & \textbf{0.05} & \textbf{0.0} \\
\hline\hline
\multirow{5}{*}{\textbf{mul32}\cite{yin2025boole}}& \multirow{5}{*}{{44557}} & \multirow{5}{*}{{293.31}} & \multirow{5}{*}{{1.56}} & ILP & $\infty$ & $\infty$ & $\infty$ & $\infty$ & $\infty$ & ILP & $\infty$ & $\infty$ & $\infty$ & $\infty$ & $\infty$ & ILP & 175.71 & 266.53 & 300.79 & 301.13 & $\infty$\\
 & &  & & 1 & $\infty$ & $\infty$ & $\infty$ & $\infty$ & $\infty$ & 1 & $\infty$ & $\infty$ & $\infty$ & $\infty$ & $\infty$ & 1 & $\infty$ & $\infty$ & $\infty$ & $\infty$ & $\infty$\\
 & &  & & 1.05 & $\infty$ & $\infty$ & $\infty$ & $\infty$ & $\infty$ & 1.05 & $\infty$ & $\infty$ & $\infty$ & $\infty$ & $\infty$ & 1.05 & $\infty$ & $\infty$ & $\infty$ & $\infty$ & $\infty$\\
 & &  & & 1.25 & $\infty$ & $\infty$ & $\infty$ & $\infty$ & $\infty$ & 1.25 & 109.73 & 109.76 & 240.98 & 244.5 & $\infty$ & 1.25 & 118.26 & 126.74 & 138.97 & 150.33 & \color{red}{200.34}\\
 & &  & & 1.5 & $\infty$ & $\infty$ & $\infty$ & $\infty$ & $\infty$ & 1.5 & 132.22 & 132.30 & 223.9 & 313.52 & $\infty$ & 1.5 & 140.7 & 148.07 & 174.79 & 189.85 & $\infty$\\
\hline
\multirow{5}{*}{\textbf{mul48}\cite{yin2025boole}}& \multirow{5}{*}{{102157}} & \multirow{5}{*}{{655.46}} & \multirow{5}{*}{{2.53}} & ILP & $\infty$ & $\infty$ & $\infty$ & $\infty$ & $\infty$ & ILP & $\infty$ & $\infty$ & $\infty$ & $\infty$ & $\infty$ & ILP & 3404.15 & $\infty$ & $\infty$ & $\infty$ & $\infty$\\
 & &  & & 1 & $\infty$ & $\infty$ & $\infty$ & $\infty$ & $\infty$ & 1 & $\infty$ & $\infty$ & $\infty$ & $\infty$ & $\infty$ & 1 & $\infty$ & $\infty$ & $\infty$ & $\infty$ & $\infty$\\
 & &  & & 1.05 & $\infty$ & $\infty$ & $\infty$ & $\infty$ & $\infty$ & 1.05 & $\infty$ & $\infty$ & $\infty$ & $\infty$ & $\infty$ & 1.05 & $\infty$ & $\infty$ & $\infty$ & $\infty$ & $\infty$\\
 & &  & & 1.25 & $\infty$ & $\infty$ & $\infty$ & $\infty$ & $\infty$ & 1.25 & 1698.91 & 1713.66 & 1804.44 & 2015.79 & $\infty$ & 1.25 & 455.72 & 504.37 & 525.2 & 533.28 & 852.37\\
 & &  & & 1.5 & $\infty$ & $\infty$ & $\infty$ & $\infty$ & $\infty$ & 1.5 & 1622.52 & 1642.48 & 1816.35 & 1964.49 & 3600.83 & 1.5 & 272.22 & 357.16 & 374.23 & 374.23 & 831.97\\
\hline
\multirow{5}{*}{\textbf{mul64}\cite{yin2025boole}}& \multirow{5}{*}{{183309}} & \multirow{5}{*}{{1045.86}} & \multirow{5}{*}{{2.98}} & ILP & $\infty$ & $\infty$ & $\infty$ & $\infty$ & $\infty$ & ILP & $\infty$ & $\infty$ & $\infty$ & $\infty$ & $\infty$ & ILP & $\infty$ & $\infty$ & $\infty$ & $\infty$ & $\infty$\\
 & &  & & 1 & $\infty$ & $\infty$ & $\infty$ & $\infty$ & $\infty$ & 1 & $\infty$ & $\infty$ & $\infty$ & $\infty$ & $\infty$ & 1 & $\infty$ & $\infty$ & $\infty$ & $\infty$ & $\infty$\\
 & &  & & 1.05 & $\infty$ & $\infty$ & $\infty$ & $\infty$ & $\infty$ & 1.05 & $\infty$ & $\infty$ & $\infty$ & $\infty$ & $\infty$ & 1.05 & $\infty$ & $\infty$ & $\infty$ & $\infty$ & $\infty$\\
 & &  & & 1.25 & $\infty$ & $\infty$ & $\infty$ & $\infty$ & $\infty$ & 1.25 & $\infty$ & $\infty$ & $\infty$ & $\infty$ & $\infty$ & 1.25 & 527.34 & 573.32 & 625.22 & 662.29 & 1206.71\\
 & &  & & 1.5 & $\infty$ & $\infty$ & $\infty$ & $\infty$ & $\infty$ & 1.5 & 2821.94 & 2822.98 & 3307.71 & 3345.4 & $\infty$ & 1.5 & 518.21 & 726.81 & 855.23 & 899.05 & 1944.61\\
\hline
\multirow{5}{*}{\textbf{mul80}\cite{yin2025boole}}& \multirow{5}{*}{{288013}} & \multirow{5}{*}{{1206.88}} & \multirow{5}{*}{{3.83}} & ILP & $\infty$ & $\infty$ & $\infty$ & $\infty$ & $\infty$ & ILP & $\infty$ & $\infty$ & $\infty$ & $\infty$ & $\infty$ & ILP & $\infty$ & $\infty$ & $\infty$ & $\infty$ & $\infty$\\
 & &  & & 1 & $\infty$ & $\infty$ & $\infty$ & $\infty$ & $\infty$ & 1 & $\infty$ & $\infty$ & $\infty$ & $\infty$ & $\infty$ & 1 & $\infty$ & $\infty$ & $\infty$ & $\infty$ & $\infty$\\
 & &  & & 1.05 & $\infty$ & $\infty$ & $\infty$ & $\infty$ & $\infty$ & 1.05 & $\infty$ & $\infty$ & $\infty$ & $\infty$ & $\infty$ & 1.05 & $\infty$ & $\infty$ & $\infty$ & $\infty$ & $\infty$\\
 & &  & & 1.25 & $\infty$ & $\infty$ & $\infty$ & $\infty$ & $\infty$ & 1.25 & $\infty$ & $\infty$ & $\infty$ & $\infty$ & $\infty$ & 1.25 & 527.62 & 590.55 & 768.55 & 826.88 & 1886.97\\
 & &  & & 1.5 & $\infty$ & $\infty$ & $\infty$ & $\infty$ & $\infty$ & 1.5 & $\infty$ & $\infty$ & $\infty$ & $\infty$ & $\infty$ & 1.5 & 1312.57 & 1400.71 & 1586.9 & 1762.83 & 2464.68\\
\hline
\multirow{5}{*}{\textbf{mul96}\cite{yin2025boole}}& \multirow{5}{*}{{416269}} & \multirow{5}{*}{{OOM}} & \multirow{5}{*}{{OOM}} & ILP & $\infty$ & $\infty$ & $\infty$ & $\infty$ & $\infty$ & ILP & $\infty$ & $\infty$ & $\infty$ & $\infty$ & $\infty$ & ILP & $\infty$ & $\infty$ & $\infty$ & $\infty$ & $\infty$\\
 & &  & & 1 & $\infty$ & $\infty$ & $\infty$ & $\infty$ & $\infty$ & 1 & $\infty$ & $\infty$ & $\infty$ & $\infty$ & $\infty$ & 1 & $\infty$ & $\infty$ & $\infty$ & $\infty$ & $\infty$\\
 & &  & & 1.05 & $\infty$ & $\infty$ & $\infty$ & $\infty$ & $\infty$ & 1.05 & $\infty$ & $\infty$ & $\infty$ & $\infty$ & $\infty$ & 1.05 & $\infty$ & $\infty$ & $\infty$ & $\infty$ & $\infty$\\
 & &  & & 1.25 & $\infty$ & $\infty$ & $\infty$ & $\infty$ & $\infty$ & 1.25 & $\infty$ & $\infty$ & $\infty$ & $\infty$ & $\infty$ & 1.25 & 859.12 & 1110.24 & 1454.32 & 1556.26 & 3035.34\\
 & &  & & 1.5 & $\infty$ & $\infty$ & $\infty$ & $\infty$ & $\infty$ & 1.5 & $\infty$ & $\infty$ & $\infty$ & $\infty$ & $\infty$ & 1.5 & 1481.2 & 1974.54 & 2277.46 & 2408.68 & $\infty$\\
\hline
\multirow{5}{*}{\textbf{mul128}\cite{yin2025boole}}& \multirow{5}{*}{{743437}} & \multirow{5}{*}{{OOM}} & \multirow{5}{*}{{OOM}} & ILP & $\infty$ & $\infty$ & $\infty$ & $\infty$ & $\infty$ & ILP & $\infty$ & $\infty$ & $\infty$ & $\infty$ & $\infty$ & ILP & $\infty$ & $\infty$ & $\infty$ & $\infty$ & $\infty$\\
 & &  & & 1 & $\infty$ & $\infty$ & $\infty$ & $\infty$ & $\infty$ & 1 & $\infty$ & $\infty$ & $\infty$ & $\infty$ & $\infty$ & 1 & $\infty$ & $\infty$ & $\infty$ & $\infty$ & $\infty$\\
 & &  & & 1.05 & $\infty$ & $\infty$ & $\infty$ & $\infty$ & $\infty$ & 1.05 & $\infty$ & $\infty$ & $\infty$ & $\infty$ & $\infty$ & 1.05 & $\infty$ & $\infty$ & $\infty$ & $\infty$ & $\infty$\\
 & &  & & 1.25 & $\infty$ & $\infty$ & $\infty$ & $\infty$ & $\infty$ & 1.25 & $\infty$ & $\infty$ & $\infty$ & $\infty$ & $\infty$ & 1.25 & $\infty$ & $\infty$ & $\infty$ & $\infty$ & $\infty$\\
 & &  & & 1.5 & $\infty$ & $\infty$ & $\infty$ & $\infty$ & $\infty$ & 1.5 & $\infty$ & $\infty$ & $\infty$ & $\infty$ & $\infty$ & 1.5 & 2380.14 & 2725.65 & 3111.41 & 3292.51 & 3599.09\\
\hline
\multirow{5}{*}{\textbf{mul32\_map}\cite{yin2025boole}}& \multirow{5}{*}{{73442}} & \multirow{5}{*}{{417.93}} & \multirow{5}{*}{{1.89}} & ILP & 28.77 & 28.77 & 28.77 & 28.77 & 46.69 & ILP & 403.77 & 403.77 & 403.77 & 403.77 & 461.08 & ILP & 79.63 & 81.85 & 82.29 & 82.29 & 82.29\\
 & &  & & 1 & 6.2 & $\infty$ & $\infty$ & $\infty$ & $\infty$ & 1 & 2.67 & $\infty$ & $\infty$ & $\infty$ & $\infty$ & 1 & 34.39 & $\infty$ & $\infty$ & $\infty$ & $\infty$\\
 & &  & & 1.05 & 21.86 & 21.86 & 21.86 & 21.86 & $\infty$ & 1.05 & 19.2 & 19.22 & 22.79 & 22.79 & $\infty$ & 1.05 & 19.64 & 19.71 & 19.71 & 19.71 & $\infty$\\
 & &  & & 1.25 & 20.13 & 20.13 & 20.13 & 20.13 & 26.08 & 1.25 & 22.45 & 22.98 & 27.74 & 27.74 & 28.98 & 1.25 & 20.15 & 20.2 & 20.2 & 20.2 & 20.49\\
 & &  & & 1.5 & 45.95 & 45.95 & 45.95 & 45.95 & 52.4 & 1.5 & 50.99 & 52.09 & 65.52 & 66.78 & 96.65 & 1.5 & 44.28 & 44.87 & 44.87 & 44.93 & 45.87\\
\hline
\multirow{5}{*}{\textbf{mul48\_map}\cite{yin2025boole}}& \multirow{5}{*}{{168940}} & \multirow{5}{*}{{709.44}} & \multirow{5}{*}{{4.32}} & ILP & 286.92 & 286.92 & 286.92 & 286.92 & 305.47 & ILP & $\infty$ & $\infty$ & $\infty$ & $\infty$ & $\infty$ & ILP & 223.53 & 224.22 & 224.22 & 224.22 & 224.7\\
 & &  & & 1 & 24.64 & $\infty$ & $\infty$ & $\infty$ & $\infty$ & 1 & 7.86 & $\infty$ & $\infty$ & $\infty$ & $\infty$ & 1 & 112.46 & $\infty$ & $\infty$ & $\infty$ & $\infty$\\
 & &  & & 1.05 & 120.58 & 120.58 & 120.58 & 120.58 & $\infty$ & 1.05 & 57.52 & 71.02 & 71.02 & 71.02 & $\infty$ & 1.05 & 80.63 & 81.58 & 81.58 & 81.64 & $\infty$\\
 & &  & & 1.25 & 319.75 & 319.75 & 319.75 & 319.75 & 319.75 & 1.25 & 72.08 & 72.08 & 86.78 & 86.78 & 142.95 & 1.25 & 82.03 & 83.17 & 83.17 & 83.28 & 84.36\\
 & &  & & 1.5 & 333.4 & 333.4 & 333.4 & 333.4 & 333.4 & 1.5 & 154.56 & 154.56 & 157.16 & 175.07 & 431.18 & 1.5 & 161.12 & 162.86 & 162.86 & 162.86 & 166.51\\
\hline
\multirow{5}{*}{\textbf{mul64\_map}\cite{yin2025boole}}& \multirow{5}{*}{{303327}} & \multirow{5}{*}{{1206.8}} & \multirow{5}{*}{{6.37}} & ILP & 3081.05 & 3081.05 & 3081.05 & 3081.05 & $\infty$ & ILP & $\infty$ & $\infty$ & $\infty$ & $\infty$ & $\infty$ & ILP & 471.58 & 473.81 & 474.67 & 474.67 & 476.65\\
 & &  & & 1 & 86.19 & $\infty$ & $\infty$ & $\infty$ & $\infty$ & 1 & 21.69 & $\infty$ & $\infty$ & $\infty$ & $\infty$ & 1 & 282.55 & $\infty$ & $\infty$ & $\infty$ & $\infty$\\
 & &  & & 1.05 & 696.22 & 696.22 & 696.22 & 696.22 & $\infty$ & 1.05 & 110.73 & 110.73 & 110.73 & 110.73 & $\infty$ & 1.05 & 200.05 & 202.4 & 202.4 & 202.48 & $\infty$\\
 & &  & & 1.25 & 688.63 & 688.63 & 688.63 & 688.63 & 2964.66 & 1.25 & 177.02 & 177.02 & 177.02 & 215.2 & 399.1 & 1.25 & 194.81 & 197.09 & 197.09 & 197.09 & 197.45\\
 & &  & & 1.5 & 1074.56 & 1074.56 & 1074.56 & 1074.56 & 3448.67 & 1.5 & 305.59 & 305.59 & 305.59 & 305.59 & 818.76 & 1.5 & 430.59 & 430.88 & 430.88 & 430.88 & 443.09\\
\hline
\multirow{5}{*}{\textbf{mul80\_map}\cite{yin2025boole}}& \multirow{5}{*}{{474818}} & \multirow{5}{*}{{OOM}} & \multirow{5}{*}{{OOM}} & ILP & $\infty$ & $\infty$ & $\infty$ & $\infty$ & $\infty$ & ILP & $\infty$ & $\infty$ & $\infty$ & $\infty$ & $\infty$ & ILP & 438.15 & 441.76 & 441.76 & 441.76 & 447.1\\
 & &  & & 1 & 225.83 & $\infty$ & $\infty$ & $\infty$ & $\infty$ & 1 & 45.45 & $\infty$ & $\infty$ & $\infty$ & $\infty$ & 1 & 420.67 & $\infty$ & $\infty$ & $\infty$ & $\infty$\\
 & &  & & 1.05 & $\infty$ & $\infty$ & $\infty$ & $\infty$ & $\infty$ & 1.05 & 394.83 & 394.83 & 505.96 & 505.96 & $\infty$ & 1.05 & 407.05 & 409.68 & 409.68 & 409.68 & $\infty$\\
 & &  & & 1.25 & $\infty$ & $\infty$ & $\infty$ & $\infty$ & $\infty$ & 1.25 & 316.54 & 330.24 & 388.86 & 388.86 & $\infty$ & 1.25 & 392.91 & 395.78 & 398.21 & 398.21 & 402.69\\
 & &  & & 1.5 & $\infty$ & $\infty$ & $\infty$ & $\infty$ & $\infty$ & 1.5 & 589.63 & 589.63 & 663.64 & 663.64 & 1071.17 & 1.5 & 1005.54 & 1014.32 & 1016.99 & 1016.99 & 1016.99\\
\hline
\multirow{5}{*}{\textbf{mul96\_map}\cite{yin2025boole}}& \multirow{5}{*}{{683229}} & \multirow{5}{*}{{OOM}} & \multirow{5}{*}{{OOM}} & ILP & $\infty$ & $\infty$ & $\infty$ & $\infty$ & $\infty$ & ILP & $\infty$ & $\infty$ & $\infty$ & $\infty$ & $\infty$ & ILP & $\infty$ & $\infty$ & $\infty$ & $\infty$ & $\infty$\\
 & &  & & 1 & 630.01 & 630.01 & $\infty$ & $\infty$ & $\infty$ & 1 & 81.46 & 81.46 & $\infty$ & $\infty$ & $\infty$ & 1 & 404.82 & 404.82 & $\infty$ & $\infty$ & $\infty$\\
 & &  & & 1.05 & $\infty$ & $\infty$ & $\infty$ & $\infty$ & $\infty$ & 1.05 & 600.88 & 775.09 & 775.09 & 775.34 & $\infty$ & 1.05 & 1275.57 & 1287.68 & 1288.05 & 1288.05 & $\infty$\\
 & &  & & 1.25 & $\infty$ & $\infty$ & $\infty$ & $\infty$ & $\infty$ & 1.25 & 516.39 & 707.13 & 707.13 & 707.4 & 871.08 & 1.25 & 1303.7 & 1311.77 & 1311.77 & 1311.77 & 1317.32\\
 & &  & & 1.5 & $\infty$ & $\infty$ & $\infty$ & $\infty$ & $\infty$ & 1.5 & 563.09 & 846.15 & 846.15 & 1042.31 & 1042.31 & 1.5 & 1317.71 & 1322.23 & 1322.23 & 1322.23 & 1322.6\\
\hline
\multirow{5}{*}{\textbf{mul128\_map}\cite{yin2025boole}}& \multirow{5}{*}{{1212206}} & \multirow{5}{*}{{OOM}} & \multirow{5}{*}{{OOM}} & ILP & $\infty$ & $\infty$ & $\infty$ & $\infty$ & $\infty$ & ILP & $\infty$ & $\infty$ & $\infty$ & $\infty$ & $\infty$ & ILP & $\infty$ & $\infty$ & $\infty$ & $\infty$ & $\infty$\\
 & &  & & 1 & 3346.03 & 3346.03 & $\infty$ & $\infty$ & $\infty$ & 1 & 175.52 & 175.52 & $\infty$ & $\infty$ & $\infty$ & 1 & 2030.47 & 2030.47 & $\infty$ & $\infty$ & $\infty$\\
 & &  & & 1.05 & $\infty$ & $\infty$ & $\infty$ & $\infty$ & $\infty$ & 1.05 & 3376.35 & 3376.35 & 3376.35 & $\infty$ & $\infty$ & 1.05 & 1965.69 & 1980.49 & 2006.79 & 2015.12 & $\infty$\\
 & &  & & 1.25 & $\infty$ & $\infty$ & $\infty$ & $\infty$ & $\infty$ & 1.25 & $\infty$ & $\infty$ & $\infty$ & $\infty$ & $\infty$ & 1.25 & 2299.12 & 2299.12 & 2299.12 & 2299.12 & 3442.33\\
 & &  & & 1.5 & $\infty$ & $\infty$ & $\infty$ & $\infty$ & $\infty$ & 1.5 & $\infty$ & $\infty$ & $\infty$ & $\infty$ & $\infty$ & 1.5 & 1915.27 & 1917.55 & 1931.28 & 1931.28 & 2031.06\\
\hline
\multirow{5}{*}{\textbf{c2670}\cite{chen2024syn}}& \multirow{5}{*}{{13205}} & \multirow{5}{*}{{34.74}} & \multirow{5}{*}{{0.5}} & ILP & $\infty$ & $\infty$ & $\infty$ & $\infty$ & $\infty$ & ILP & 1042.62 & 1303.82 & 1372.91 & 1488.37 & $\infty$ & ILP & 202.17 & 300.43 & 390.93 & 462.33 & 679.8\\
 & &  & & 1 & $\infty$ & $\infty$ & $\infty$ & $\infty$ & $\infty$ & 1 & $\infty$ & $\infty$ & $\infty$ & $\infty$ & $\infty$ & 1 & $\infty$ & $\infty$ & $\infty$ & $\infty$ & $\infty$\\
 & &  & & 1.05 & 715.49 & $\infty$ & $\infty$ & $\infty$ & $\infty$ & 1.05 & 22.24 & 96.07 & $\infty$ & $\infty$ & $\infty$ & 1.05 & 26.44 & 31.28 & 44.21 & 58.47 & 530.4\\
 & &  & & 1.25 & 1064.72 & 1771.6 & $\infty$ & $\infty$ & $\infty$ & 1.25 & 69.61 & 107.81 & 138.02 & 161.56 & $\infty$ & 1.25 & 33.78 & 62.57 & 116.09 & 156.14 & 1055.68\\
 & &  & & 1.5 & 1967.89 & $\infty$ & $\infty$ & $\infty$ & $\infty$ & 1.5 & 63.8 & 118.94 & 186.31 & 202.92 & $\infty$ & 1.5 & 38.96 & 98.36 & 144.71 & 181.03 & 667.24\\
\hline
\multirow{5}{*}{\textbf{qdiv}\cite{chen2024syn}}& \multirow{5}{*}{{14604}} & \multirow{5}{*}{{55.15}} & \multirow{5}{*}{{0.26}} & ILP & 39.08 & 39.08 & 39.08 & 39.08 & 546.79 & ILP & 99.93 & 99.93 & 99.93 & 99.93 & 100.1 & ILP & 25.62 & 34.59 & 70.11 & 76.68 & 156.78\\
 & &  & & 1 & $\infty$ & $\infty$ & $\infty$ & $\infty$ & $\infty$ & 1 & $\infty$ & $\infty$ & $\infty$ & $\infty$ & $\infty$ & 1 & 5.87 & 6.91 & 10.53 & 26.62 & 157.16\\
 & &  & & 1.05 & 1.13 & $\infty$ & $\infty$ & $\infty$ & $\infty$ & 1.05 & 1.13 & $\infty$ & $\infty$ & $\infty$ & $\infty$ & 1.05 & 12.12 & 16.97 & 20.75 & 27.58 & 136.53\\
 & &  & & 1.25 & 1.25 & 1.25 & 1.25 & 1.25 & 1.57 & 1.25 & 1.04 & 1.18 & 1.28 & 1.28 & 1.66 & 1.25 & 12.6 & 16.83 & 27.31 & 40.46 & 90.28\\
 & &  & & 1.5 & 1.35 & 1.35 & 1.35 & 1.35 & 1.57 & 1.5 & 1.4 & 1.53 & 1.69 & 1.69 & 2.36 & 1.5 & 25.34 & 28.51 & 34.38 & 42.2 & 149.89\\
\hline
\multirow{5}{*}{\textbf{sin}\cite{chen2025emorphicscalableequalitysaturation}}& \multirow{5}{*}{{20390}} & \multirow{5}{*}{{69.39}} & \multirow{5}{*}{{0.65}} & ILP & 247.61 & 247.61 & 247.61 & 247.61 & 3331.98 & ILP & 35.43 & 35.43 & 59.49 & 63.6 & 211.9 & ILP & 12.59 & 27.39 & 55.55 & 66.11 & 390.94\\
 & &  & & 1 & 1.26 & $\infty$ & $\infty$ & $\infty$ & $\infty$ & 1 & 2.02 & $\infty$ & $\infty$ & $\infty$ & $\infty$ & 1 & 7.99 & $\infty$ & $\infty$ & $\infty$ & $\infty$\\
 & &  & & 1.05 & 95.96 & 95.96 & 161.54 & $\infty$ & $\infty$ & 1.05 & 5.59 & 11.91 & 31.77 & $\infty$ & $\infty$ & 1.05 & 56.53 & 117.08 & 365.62 & $\infty$ & $\infty$\\
 & &  & & 1.25 & 60.88 & 60.88 & 77.46 & $\infty$ & $\infty$ & 1.25 & 7.99 & 15.42 & 35.89 & $\infty$ & $\infty$ & 1.25 & 100.07 & 147.61 & 463.21 & $\infty$ & $\infty$\\
 & &  & & 1.5 & 55.13 & 55.13 & 70.33 & $\infty$ & $\infty$ & 1.5 & 6.12 & 17.78 & 43.36 & $\infty$ & $\infty$ & 1.5 & 95.31 & 126.73 & 314.65 & $\infty$ & $\infty$\\
\hline
\multirow{5}{*}{\textbf{log2}\cite{chen2025emorphicscalableequalitysaturation}}& \multirow{5}{*}{{142619}} & \multirow{5}{*}{{169.82}} & \multirow{5}{*}{{0.03}} & ILP & $\infty$ & $\infty$ & $\infty$ & $\infty$ & $\infty$ & ILP & 155.9 & 266.33 & 293.02 & 563.04 & 2044.11 & ILP & 427.7 & 543.95 & 1591.54 & 3169.89 & $\infty$\\
 & &  & & 1 & 1066.47 & $\infty$ & $\infty$ & $\infty$ & $\infty$ & 1 & 34.24 & $\infty$ & $\infty$ & $\infty$ & $\infty$ & 1 & 140.8 & $\infty$ & $\infty$ & $\infty$ & $\infty$\\
 & &  & & 1.05 & $\infty$ & $\infty$ & $\infty$ & $\infty$ & $\infty$ & 1.05 & 60.24 & 210.72 & 283.74 & 479.15 & $\infty$ & 1.05 & 315.45 & 327.57 & 430.61 & 1379.85 & $\infty$\\
 & &  & & 1.25 & $\infty$ & $\infty$ & $\infty$ & $\infty$ & $\infty$ & 1.25 & 58.04 & 183.3 & 214.79 & 426.36 & $\infty$ & 1.25 & 447.42 & 468.9 & 517.53 & 1242.7 & $\infty$\\
 & &  & & 1.5 & $\infty$ & $\infty$ & $\infty$ & $\infty$ & $\infty$ & 1.5 & 120.5 & 253.42 & 379.24 & 503.5 & 1167.54 & 1.5 & 655.04 & 691.49 & 762.72 & 1106.6 & 2244.18\\
\hline
\multirow{5}{*}{\textbf{adder}\cite{chen2025emorphicscalableequalitysaturation}}& \multirow{5}{*}{{29701}} & \multirow{5}{*}{{98.42}} & \multirow{5}{*}{{4.79}} & ILP & $\infty$ & $\infty$ & $\infty$ & $\infty$ & $\infty$ & ILP & 174.0 & 206.06 & 812.98 & 1404.58 & 1731.67 & ILP & 337.65 & 424.55 & 511.69 & 730.48 & 933.25\\
 & &  & & 1 & $\infty$ & $\infty$ & $\infty$ & $\infty$ & $\infty$ & 1 & $\infty$ & $\infty$ & $\infty$ & $\infty$ & $\infty$ & 1 & $\infty$ & $\infty$ & $\infty$ & $\infty$ & $\infty$\\
 & &  & & 1.05 & 44.62 & 44.62 & 44.73 & 56.25 & $\infty$ & 1.05 & 3.92 & 3.92 & 4.87 & 4.87 & $\infty$ & 1.05 & 64.28 & 77.17 & 86.87 & 90.51 & $\infty$\\
 & &  & & 1.25 & 474.78 & 531.76 & 712.21 & 967.25 & $\infty$ & 1.25 & 8.38 & 11.15 & 15.21 & 17.84 & 33.88 & 1.25 & 138.78 & 152.32 & 169.29 & 191.17 & 347.24\\
 & &  & & 1.5 & 298.1 & 671.26 & 906.79 & 1002.94 & $\infty$ & 1.5 & 27.32 & 29.46 & 29.47 & 44.66 & 64.3 & 1.5 & 85.61 & 116.84 & 134.86 & 143.41 & 167.53\\
\hline
\multirow{5}{*}{\textbf{fir\_8\_tap7}\cite{coward2022automatic}}& \multirow{5}{*}{{12224}} & \multirow{5}{*}{{4.03}} & \multirow{5}{*}{{0}} & ILP & $\infty$ & $\infty$ & $\infty$ & $\infty$ & $\infty$ & ILP & 43.37 & 43.37 & 43.37 & 43.37 & 43.37 & ILP & 26.84 & 26.84 & 26.84 & 26.84 & 26.84\\
 & &  & & 1 & 0.69 & 0.69 & 0.69 & 0.69 & 0.69 & 1 & 0.56 & 0.56 & 0.56 & 0.56 & 0.56 & 1 & 0.93 & 0.93 & 0.93 & 0.93 & 0.93\\
 & &  & & 1.05 & 1.0 & 1.0 & 1.0 & 1.0 & 1.0 & 1.05 & 0.12 & 0.12 & 0.12 & 0.12 & 0.12 & 1.05 & 4.19 & 4.19 & 4.19 & 4.19 & 4.19\\
 & &  & & 1.25 & 0.83 & 0.83 & 0.83 & 0.83 & 0.83 & 1.25 & 0.19 & 0.19 & 0.19 & 0.19 & 0.19 & 1.25 & 13.1 & 13.1 & 13.1 & 13.1 & 13.1\\
 & &  & & 1.5 & 0.85 & 0.85 & 0.85 & 0.85 & 0.85 & 1.5 & 0.08 & 0.08 & 0.08 & 0.08 & 0.08 & 1.5 & 11.21 & 11.21 & 11.21 & 11.21 & 11.21\\
\hline
\multirow{5}{*}{\textbf{direct\_root\_18}\cite{vanhattum2021vectorization}}& \multirow{5}{*}{{218933}} & \multirow{5}{*}{{41.42}} & \multirow{5}{*}{{0}} & ILP & 144.38 & 144.38 & 144.38 & 144.38 & 144.38 & ILP & 180.72 & 180.72 & 180.72 & 180.72 & 180.72 & ILP & 92.22 & 92.22 & 92.22 & 92.22 & 92.22\\
 & &  & & 1 & 2.7 & 2.7 & 2.7 & 2.7 & 2.7 & 1 & 1.08 & 1.08 & 1.08 & 1.08 & 1.08 & 1 & 4.53 & 4.53 & 4.53 & 4.53 & 4.53\\
 & &  & & 1.05 & 2.38 & 2.38 & 2.38 & 2.38 & 2.38 & 1.05 & 1.05 & 1.05 & 1.05 & 1.05 & 1.05 & 1.05 & 3.82 & 3.82 & 3.82 & 3.82 & 3.82\\
 & &  & & 1.25 & 3.54 & 3.54 & 3.54 & 3.54 & 3.54 & 1.25 & 1.19 & 1.19 & 1.19 & 1.19 & 1.19 & 1.25 & 3.78 & 3.78 & 3.78 & 3.78 & 3.78\\
 & &  & & 1.5 & 2.39 & 2.39 & 2.39 & 2.39 & 2.39 & 1.5 & 1.11 & 1.11 & 1.11 & 1.11 & 1.11 & 1.5 & 4.21 & 4.21 & 4.21 & 4.21 & 4.21\\
\hline
\multirow{5}{*}{\textbf{vector\_conv\_2x2}\cite{vanhattum2021vectorization}}& \multirow{5}{*}{{3764}} & \multirow{5}{*}{{3.78}} & \multirow{5}{*}{{0}} & ILP & 0.13 & 0.13 & 0.13 & 0.13 & 0.13 & ILP & 0.0 & 0.0 & 0.0 & 0.0 & 0.0 & ILP & 0.1 & 0.1 & 0.1 & 0.1 & 0.1\\
 & &  & & 1 & 0.03 & 0.03 & 0.03 & 0.03 & 0.03 & 1 & 0.02 & 0.02 & 0.02 & 0.02 & 0.02 & 1 & 0.04 & 0.04 & 0.04 & 0.04 & 0.04\\
 & &  & & 1.05 & 0.03 & 0.03 & 0.03 & 0.03 & 0.03 & 1.05 & 0.02 & 0.02 & 0.02 & 0.02 & 0.02 & 1.05 & 0.04 & 0.04 & 0.04 & 0.04 & 0.04\\
 & &  & & 1.25 & 0.03 & 0.03 & 0.03 & 0.03 & 0.03 & 1.25 & 0.02 & 0.02 & 0.02 & 0.02 & 0.02 & 1.25 & 0.04 & 0.04 & 0.04 & 0.04 & 0.04\\
 & &  & & 1.5 & 0.03 & 0.03 & 0.03 & 0.03 & 0.03 & 1.5 & 0.02 & 0.02 & 0.02 & 0.02 & 0.02 & 1.5 & 0.04 & 0.04 & 0.04 & 0.04 & 0.04\\
\hline
\multirow{5}{*}{\textbf{large\_mul2048}\cite{ustun2022impress}} & \multirow{5}{*}{{102030}} & \multirow{5}{*}{{20.12}} & \multirow{5}{*}{{1.77}} & ILP & 19.86 & 19.86 & 19.86 & 19.86 & 19.86 & ILP & 0.2 & 0.2 & 0.2 & 0.2 & 31.49 & ILP & 15.54 & 15.54 & 15.54 & 15.54 & 15.54\\
 & &  & & 1 & 0.99 & 0.99 & 0.99 & 0.99 & 0.99 & 1 & 0.41 & 0.41 & 0.41 & 0.41 & 0.41 & 1 & 3.42 & 3.42 & 3.42 & 3.42 & 3.42\\
 & &  & & 1.05 & 1.17 & 1.17 & 1.17 & 1.17 & 1.17 & 1.05 & 0.48 & 0.48 & 0.48 & 0.48 & 0.48 & 1.05 & 4.54 & 4.54 & 4.54 & 4.54 & 4.54\\
 & &  & & 1.25 & 1.09 & 1.09 & 1.09 & 1.09 & 1.09 & 1.25 & 0.46 & 0.46 & 0.46 & 0.46 & 0.46 & 1.25 & 6.50 & 6.50 & 6.50 & 6.50 & 6.50\\
 & &  & & 1.5 & 1.1 & 1.1 & 1.1 & 1.1 & 1.1 & 1.5 & 0.46 & 0.46 & 0.46 & 0.46 & 0.46 & 1.5 & 10.32 & 10.32 & 10.32 & 10.32 & 10.32\\
\hline
\multirow{5}{*}{\textbf{nasneta}\cite{yang2021equality}}& \multirow{5}{*}{{43072}} & \multirow{5}{*}{{37.44}} & \multirow{5}{*}{{0.15}} & ILP & 16.16 & 23.11 & 23.11 & 23.11 & 27.58 & ILP & 26.1 & 26.1 & 26.1 & 26.1 & 37.38 & ILP & 39.57 & 42.93 & 60.62 & 69.72 & 69.72\\
 & &  & & 1 & $\infty$ & $\infty$ & $\infty$ & $\infty$ & $\infty$ & 1 & $\infty$ & $\infty$ & $\infty$ & $\infty$ & $\infty$ & 1 & $\infty$ & $\infty$ & $\infty$ & $\infty$ & $\infty$\\
 & &  & & 1.05 & 3.52 & $\infty$ & $\infty$ & $\infty$ & $\infty$ & 1.05 & 1.61 & $\infty$ & $\infty$ & $\infty$ & $\infty$ & 1.05 & 9.71 & $\infty$ & $\infty$ & $\infty$ & $\infty$\\
 & &  & & 1.25 & 3.41 & $\infty$ & $\infty$ & $\infty$ & $\infty$ & 1.25 & 1.17 & $\infty$ & $\infty$ & $\infty$ & $\infty$ & 1.25 & 9.32 & $\infty$ & $\infty$ & $\infty$ & $\infty$\\
 & &  & & 1.5 & 4.34 & $\infty$ & $\infty$ & $\infty$ & $\infty$ & 1.5 & 1.05 & $\infty$ & $\infty$ & $\infty$ & $\infty$ & 1.5 & 8.53 & $\infty$ & $\infty$ & $\infty$ & $\infty$\\
\hline
\end{tabular}
}
\label{tab:comprehensive_convergence}
\end{table*}

\begin{table*}[ht]
    \scriptsize
    \begin{minipage}[t]{0.48\textwidth}
        \centering
        \caption{Area Comparison with ASAP 7nm library 
        }
        \label{tab:area_comparison1}
\begin{tabular}{l|ccc}
\hline
Benchmark & Ours & E-Syn & ABC \\
\hline
adder & 927.75 & 926.59 & \textbf{\textcolor{mygreen}{882.50}} \\
b12 & 836.78 & 917.02 & \textbf{\textcolor{mygreen}{807.15}} \\
bar & 2865.61 & 2686.22 & \textbf{\textcolor{mygreen}{2666.62}} \\
c432 & \textbf{\textcolor{mygreen}{206.45}} & 233.05 & 231.88 \\
c2670 & \textbf{\textcolor{mygreen}{650.38}} & 880.63 & 700.31 \\
c5315 & \textbf{\textcolor{mygreen}{1380.32}} & 1518.19 & 1391.05 \\
c7552 & \textbf{\textcolor{mygreen}{2022.77}} & 2297.11 & 2115.15 \\
cavlc & \textbf{\textcolor{mygreen}{494.79}} & 521.38 & 500.62 \\
frg2 & \textbf{\textcolor{mygreen}{675.11}} & 703.11 & 765.16 \\
i7 & \textbf{\textcolor{mygreen}{415.01}} & 500.39 & 671.15 \\
max & \textbf{\textcolor{mygreen}{2262.82}} & 2820.36 & 2326.27 \\
qdiv & \textbf{\textcolor{mygreen}{1386.85}} & 1469.66 & 1420.44 \\
sqrt & \textbf{\textcolor{mygreen}{20322.89}} & n/a & 20578.56 \\
sin & \textbf{\textcolor{mygreen}{5220.11}} & 5437.99 & 5223.84 \\
log2 & 25969.9 & 25969.90 & \textbf{\textcolor{mygreen}{25933.74}} \\
multiplier & \textbf{\textcolor{mygreen}{23929.63}} & 25037.01 & 24787.87 \\
mem\_ctrl & \textbf{\textcolor{mygreen}{32379.73}} & 32383.70 & 32383.93 \\
square & 15214.75 & 16809.92 & \textbf{\textcolor{mygreen}{15200.06}} \\
avg. improvements & & \textbf{\textcolor{mygreen}{7.6\%}} & \textbf{\textcolor{mygreen}{3.8\%}} \\
\hline
\end{tabular}
    \end{minipage}
    \hfill
    \begin{minipage}[t]{0.48\textwidth}
        \centering
        \caption{Area Comparison with Skywater130 library  
        \cite{gdsfactory_skywater130_2025}}
        \label{tab:area_comparison2}
\begin{tabular}{l|ccc}
\hline
Benchmark & Ours & E-Syn & ABC \\
\hline
adder & \textbf{\textcolor{mygreen}{4903.45}} & 5256.29 & 5071.11 \\
b12 & \textbf{\textcolor{mygreen}{3560.92}} & 4172.75 & 3696.04 \\
bar & \textbf{\textcolor{mygreen}{10256.09}} & 10301.13 & 10996.80 \\
c432 & 1063.52 & 1093.55 & \textbf{\textcolor{mygreen}{993.45}} \\
c2670 & \textbf{\textcolor{mygreen}{2761.4}} & 3827.42 & 2910.29 \\
c5315 & \textbf{\textcolor{mygreen}{5954.46}} & 7170.63 & 6088.34 \\
c7552 & \textbf{\textcolor{mygreen}{7895.07}} & 9265.14 & 8221.63 \\
cavlc & \textbf{\textcolor{mygreen}{2209.62}} & 2524.92 & 2285.94 \\
frg2 & \textbf{\textcolor{mygreen}{3145.52}} & 3394.51 & 3609.71 \\
i7 & \textbf{\textcolor{mygreen}{2132.04}} & 2347.25 & 2140.80 \\
max & \textbf{\textcolor{mygreen}{11357.14}} & 11737.51 & 11647.42 \\
qdiv & \textbf{\textcolor{mygreen}{5904.41}} & 6248.49 & 6175.92 \\
sqrt & \textbf{\textcolor{mygreen}{124052.73}} & n/a & 127858.88 \\
sin & \textbf{\textcolor{mygreen}{22013.61}} & 23585.12 & 22291.38 \\
log2 & 111319.27 & 111319.27 & \textbf{\textcolor{mygreen}{110961.42}} \\
multiplier & \textbf{\textcolor{mygreen}{104352.58}} & 105829.00 & 106354.50 \\
mem\_ctrl & \textbf{\textcolor{mygreen}{146022.55}} & 146038.81 & 146043.81 \\
square & \textbf{\textcolor{mygreen}{66934.2}} & 72757.28 & 68964.89 \\
avg. improvements & & \textbf{\textcolor{mygreen}{8.1\%}} & \textbf{\textcolor{mygreen}{2.8\%}} \\
\hline
\end{tabular}
    \end{minipage}
    \label{fig:area_comparisons}
\end{table*}

\begin{figure*}[htbp]
   \centering
   \hfill
   \begin{subfigure}[t]{0.45\linewidth}
       \centering
       \includegraphics[width=\textwidth]{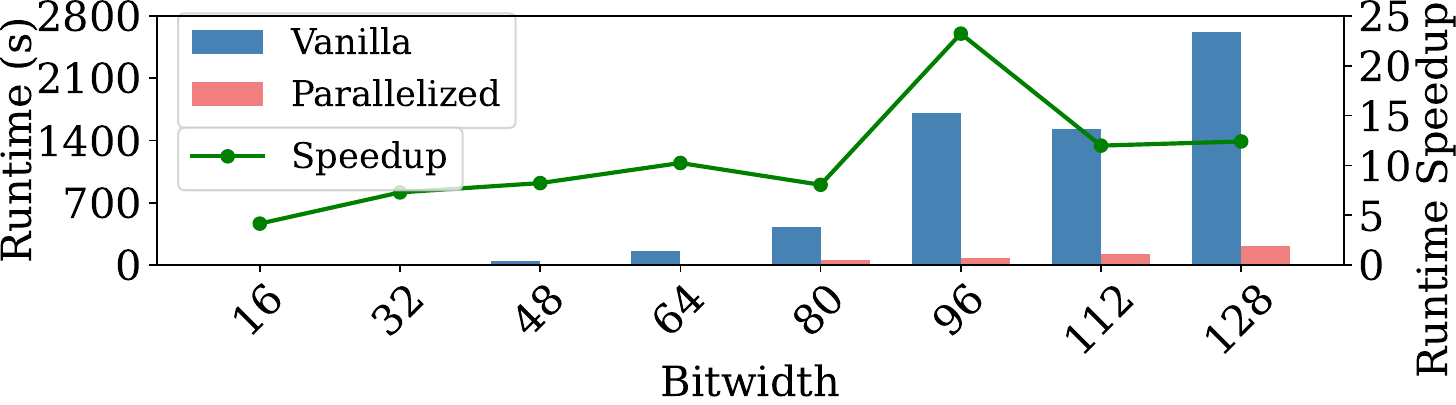}
       \label{fig:map_runtime}
   \end{subfigure}%
   \hfill
   \begin{subfigure}[t]{0.45\linewidth}
       \centering
       \includegraphics[width=\textwidth]{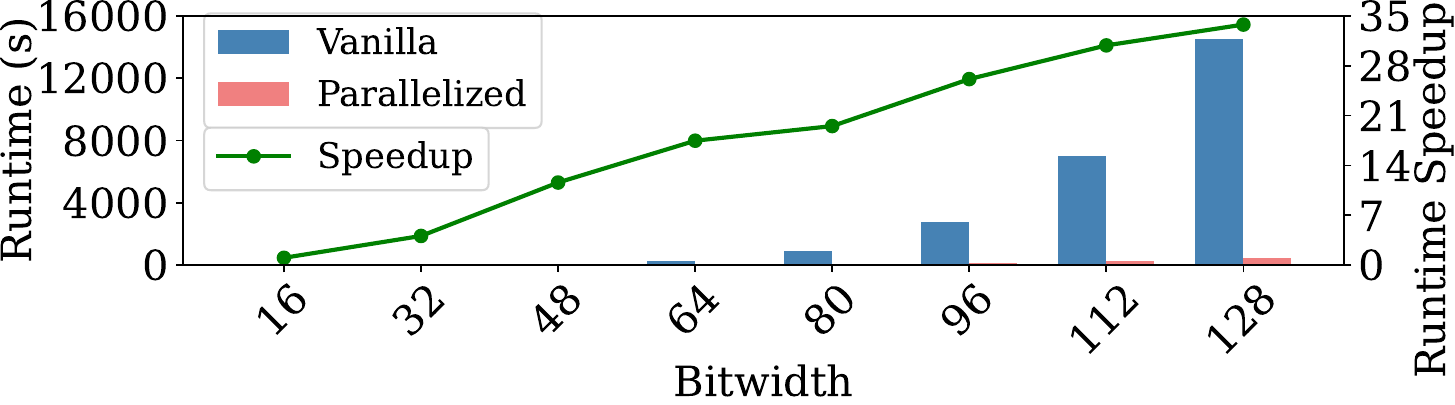}
       \label{fig:nonmap_runtime}
   \end{subfigure}%
   \hfill
   \vspace{-4mm}
   \caption{Runtime speedup comparison between vanilla and parallelized extraction. \name achieves higher speedup as input size grows. Left: non-mapped CSA multiplier; Right: mapped CSA multiplier}
   \label{fig:parallelized_performance}
   \vspace{-5mm}
\end{figure*}

Our results reveal several key insights. \textbf{First, \name demonstrates superior scalability and convergence speed.} It is the only viable approach for the largest benchmarks (BoolE\_mul96, BoolE\_mul128, etc.), where both vanilla ILP and SmoothE fail to produce any valid solution within the timeout period. Quantitatively, \name achieves $138\times$ and $558\times$ runtime speedups relative to SmoothE and vanilla exact solvers, respectively, when producing results of comparable quality. \textbf{Second, \name delivers high solution quality across diverse benchmarks.} It achieves 19.04\% and 5.6\% performance improvement compared to SmoothE and vanilla ILP solvers. For benchmarks where ILP guarantees optimality, \name maintains remarkable accuracy with only a 0.21\% optimality gap. \textbf{Finally, we observe that threshold selection represents an important trade-off parameter.} Overly aggressive threshold values can compromise solution quality by pruning optimal solutions from the search space. For instance, in the nasneta benchmark, \name fails to find optimal solutions across all threshold values ($\infty$ in Table I) because the global optimum is pruned from the search space. Across our evaluation of 92 benchmarks, \name achieves optimal results in 83 cases, with only 9 benchmarks experiencing global optima pruning. Notably, even when the global optimum is pruned, search space reduction still enables faster early-stage convergence. 


\subsection{Performance of Realistic e-graph applications}

\noindent
\textbf{RQ2: Can \name bring realistic improvement on downstream tasks?} To assess the practical impact of \name beyond abstract e-graph extraction benchmarks, we evaluated its performance within a realistic backend application: logic synthesis area optimization.

E-Syn inserts e-graph optimization at the beginning of the logic optimization flow, applying e-graph rewriting and extraction before feeding the optimized result back to ABC for technology mapping. In this experiment, we integrated e-boost into the synthesis flow of the E-Syn framework, replacing the original e-graph extraction engine with our proposed method while maintaining the same integration points with ABC. This allowed us to directly measure the influence of \name extraction approach on the final circuit area after technology mapping, and also eliminate the needs of extensive tuning for conventional logic synthesis techniques \cite{yu2020flowtune,neto2022flowtune}. We compared the results obtained using the \name enhanced flows (labeled "Ours") against two baselines: the widely-used academic logic synthesis tool ABC, and the results from the original E-Syn flow utilizing its original extraction method (labeled "E-Syn"). Note that we only performed technology-independent optimization to optimize the AIG structure, while all technology mapping operations were uniformly conducted using ABC commands to ensure fair comparison, which could be further optimized via mapper optimizations \cite{liu2024maptune,liu2025maptune,neto2021slap}. The test results are presented in Tables \ref{tab:area_comparison1} and \ref{tab:area_comparison2}, which detail the post-mapping circuit area obtained when targeting the ASAP 7nm \cite{clark2016asap7} and Skywater130 \cite{gdsfactory_skywater130_2025} technology libraries, respectively.

The results demonstrate that \name delivers substantial area improvements across both technology mapping libraries. \textbf{Overall, our approach provides 7.6\% and 8.1\% area reductions compared to the original E-Syn extractor for ASAP 7nm and Skywater130 libraries, respectively.} For critical benchmarks such as c2670, our method achieves 26.1\% area reduction (650.38 vs. 880.63) with ASAP 7nm and 27.9\% reduction (2761.40 vs. 3827.42) with Skywater130. For all tested benchmarks, \name consistently produces more compact circuit implementations than the original E-Syn approach, demonstrating that improved e-graph extraction directly translates to meaningful downstream QoR improvements. {Note that final outcomes depend on specific technology library characteristics and ABC's mapping algorithms, explaining why some benchmarks may not outperform the original ABC flow.}


\subsection{Runtime comparison of parallelized extraction}

\noindent
\textbf{RQ3: How does parallelized extraction impact runtime performance for heuristic initialization?} Figure \ref{fig:parallelized_performance} demonstrates the substantial performance improvements achieved by our parallelized extraction algorithm across BoolE benchmarks of varying bitwidths for both technology-mapped and non-technology-mapped CSA multiplier e-graphs. For non-mapped multipliers (left graph), our approach delivers speedups ranging from $4.2\times$ at 32-bit to a peak of $23.2\times$ at 96-bit. The benefits become even clearer for technology-mapped multipliers (right graph), where parallelization achieves runtime reductions—from $4.1\times$ speedup at 32-bit to $33.8\times$ at 128-bit. For example, the vanilla implementation requires 14,535 seconds for 128-bit mapped multipliers, while our parallelized version completes the same extraction in just 430 seconds. For small benchmarks (16-bit), the parallelization overhead slightly outweighs the benefits, but as problem size increases, our approach delivers substantial runtime reductions with speedups of $12.4\times$ for 128-bit CSA multiplier and $33.8\times$ for 128-bit mapped CSA multiplier.




\section{Conclusion}

In this paper, we focused on the challenge of e-graph extraction, an \textit{NP-hard} combinatorial optimization that bottlenecks equality saturation in the fields of logic synthesis and formal verification. Our \name framework balances speed and optimality via key innovations: parallelized heuristic extraction for multithreading, adaptive search-space pruning with a parameterized cost threshold, and initialized exact solving using warm-start ILP. Experiments on diverse benchmarks show \name achieves up to $558\times$ runtime speedup over vanilla ILP and 19.04\% performance gain against state-of-the-art SmoothE, achieving 7.6\% and 8.1\% area reductions in downstream logic synthesis task across two technology libraries.

\textbf{Acknowledgemnt} -- 
This work is sponsored by the National Science Foundation under \#2403134, \#2403135, \#2349670, \#2349461, \#2229562, and DARPA Grant No. HR001125C0058 and supported in part by ACE, one of the seven centers in JUMP 2.0.

\small
\bibliographystyle{plain}
\bibliography{reference}
\end{document}